\documentclass{article} 
\usepackage{iclr2026_conference,times}
\usepackage{amsthm}
\usepackage{thmtools}  
\usepackage{booktabs}
\usepackage{pifont}
\usepackage{graphicx}

\declaretheorem[name=Remark]{remark}

\usepackage{amsmath,amsfonts,bm}

\def\eqref#1{equation~\ref{#1}}

\def\1{\bm{1}}

\def\vc{{\bm{c}}}

\def\vu{{\bm{u}}}

\def\vx{{\bm{x}}}

\def\vz{{\bm{z}}}

\def\mI{{\bm{I}}}

\DeclareMathAlphabet{\mathsfit}{\encodingdefault}{\sfdefault}{m}{sl}
\SetMathAlphabet{\mathsfit}{bold}{\encodingdefault}{\sfdefault}{bx}{n}

\def\gL{{\mathcal{L}}}

\def\gN{{\mathcal{N}}}

\def\gU{{\mathcal{U}}}

\def\sR{{\mathbb{R}}}

\newcommand{\E}{\mathbb{E}}

\usepackage{hyperref}
\usepackage{url}

\usepackage{enumitem}
\usepackage{comment}
\usepackage{wrapfig}

\def\BW#1{{\textcolor{red}{#1}}}

\title{
RMFlow: Refined Mean Flow by a Noise-Injection Step for Multimodal Generation
}

\author{Yuhao Huang$^1$, Shih-Hsin Wang$^1$, Andrea L. Bertozzi$^2$, Bao Wang$^4$ \thanks{Correspond to \texttt{wangbaonj@gmail.com}} \\
$^1$Department of Mathematics and Scientific Computing and Imaging (SCI) Institute \\
University of Utah, Salt Lake City, UT, 84102, USA\\
$^2$Department of Mathematics, UCLA, Los Angeles, CA, 90095, USA\\
}

\iclrfinalcopy 
\begin{document}

\maketitle

\begin{abstract}
Mean flow (MeanFlow) enables efficient, high-fidelity image generation, yet its single-function evaluation (1-NFE) generation often cannot yield compelling results. We address this issue by introducing RMFlow, an efficient multimodal generative model that integrates a coarse 1-NFE MeanFlow transport with a subsequent tailored noise-injection refinement step. RMFlow approximates the average velocity of the flow path using a neural network trained with a new loss function that balances minimizing the Wasserstein distance between probability paths and maximizing sample likelihood. RMFlow achieves near state-of-the-art results on text-to-image, context-to-molecule, and time-series generation using only 1-NFE, at a computational cost comparable to the baseline MeanFlows.
\end{abstract}

\section{Introduction}\label{sec:intro}
\begin{wrapfigure}{r}{0.52\linewidth}
\vspace{-0.35cm}
\centering
\includegraphics[width=0.45\columnwidth]{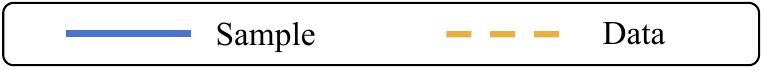}
\begin{tabular}{cc}
\includegraphics[width=0.22\columnwidth]{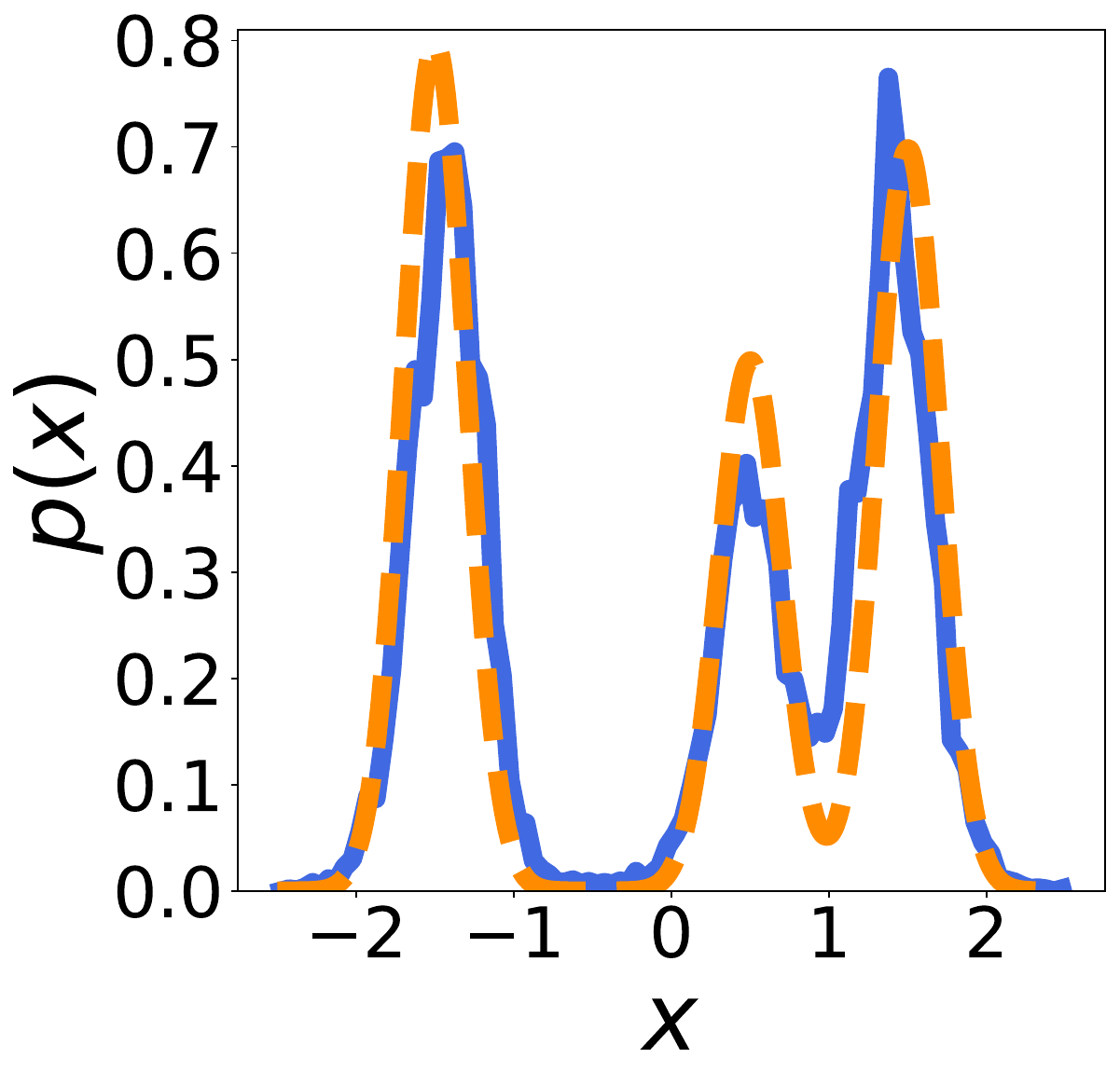}& 
\includegraphics[width=0.22\columnwidth]{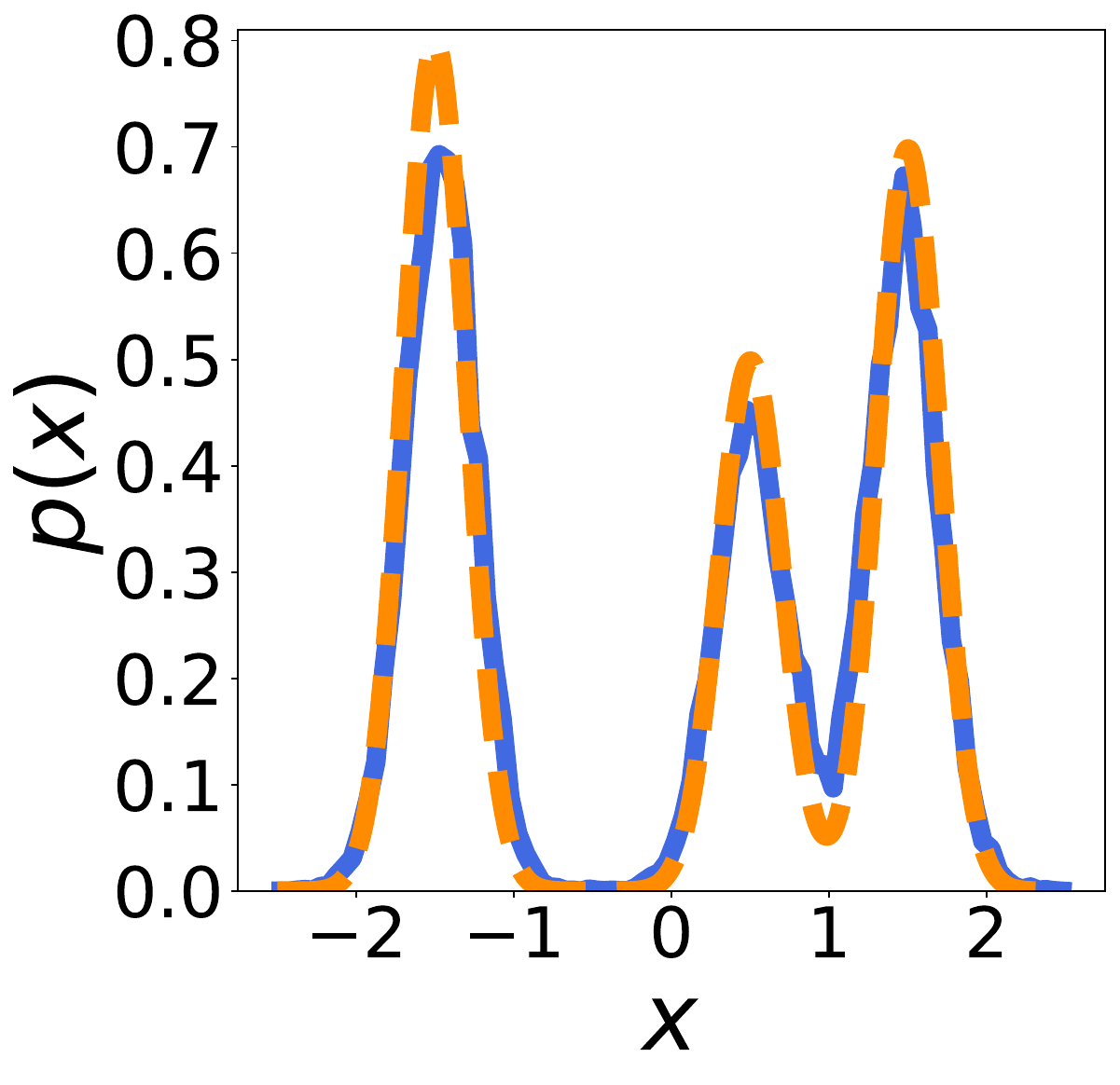}\\[-3pt]
{\scriptsize 1-NFE MeanFlow}& {\scriptsize 8-NFE MeanFlow}\\
\includegraphics[width=0.22\columnwidth]{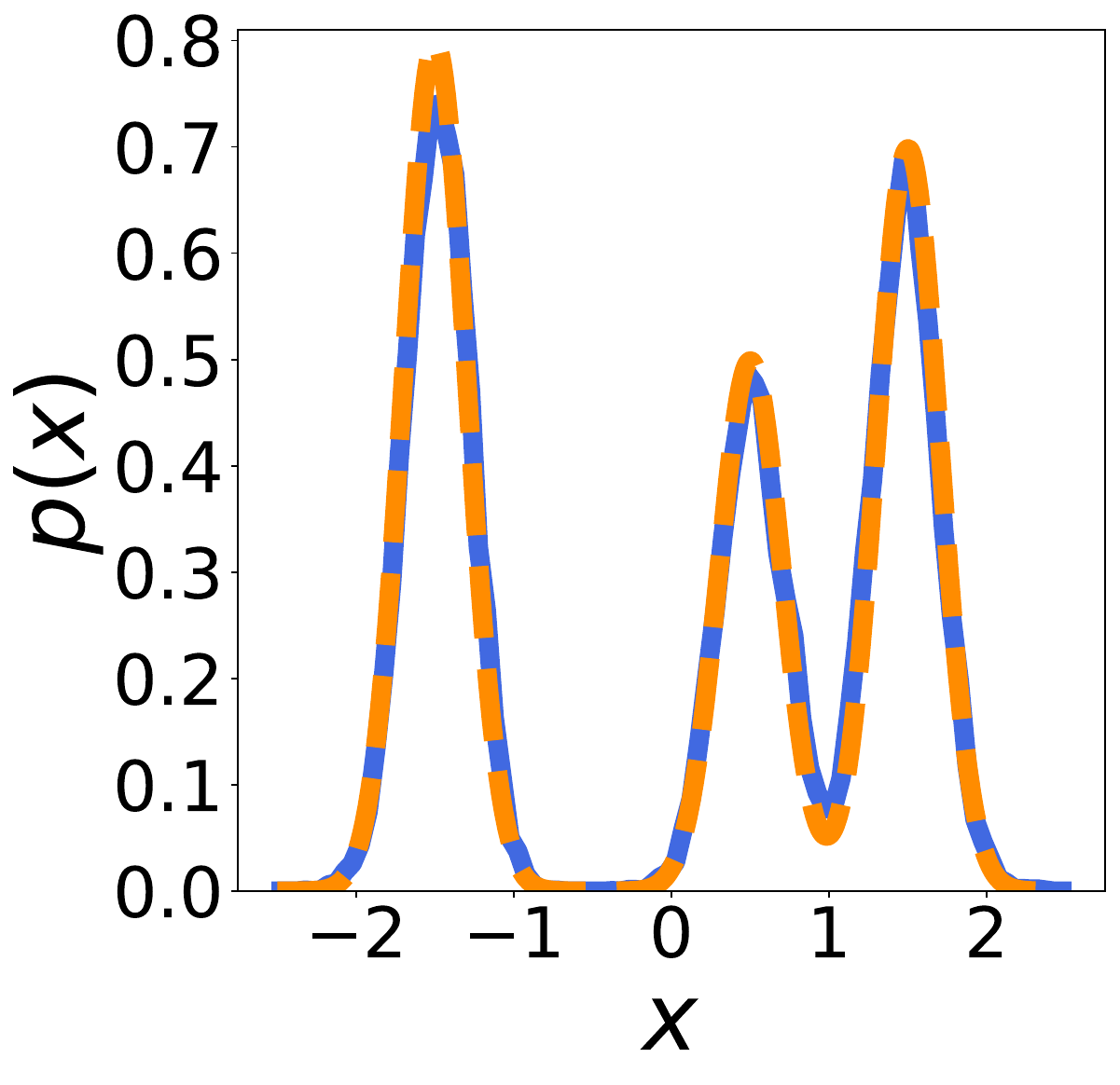}&
\includegraphics[width=0.22\columnwidth]{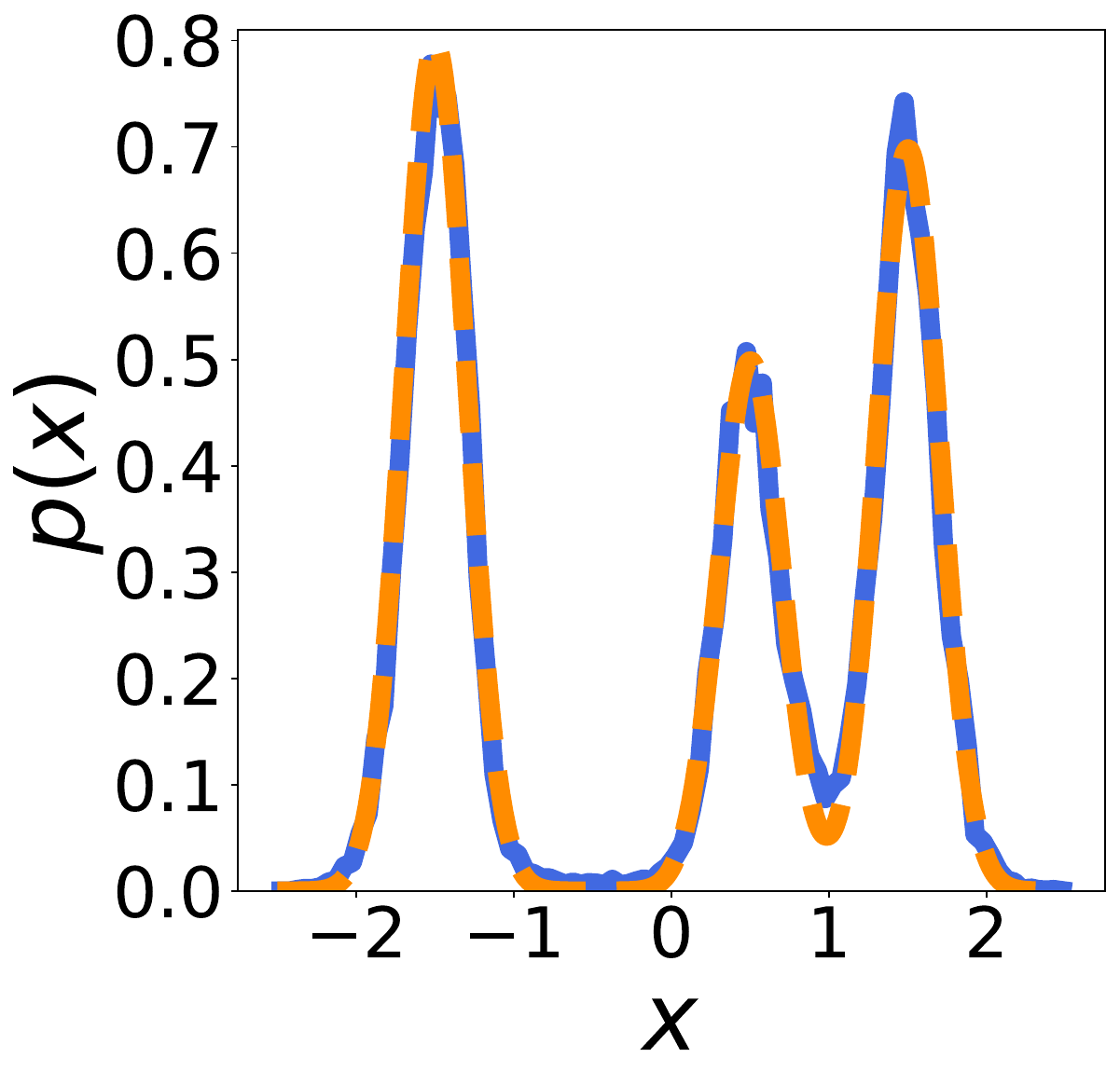}\\[-3pt]
{
32-NFE MeanFlow}& {
1-NFE RMFlow ({\bf ours})}  \\
\end{tabular}\vspace{-0.2cm}
\caption{
Contrasting MeanFlow with RMFlow for mixture Gaussian sampling; see Section~\ref{subsec:density-estimation} for experimental details and more results.}
\label{fig:rmflow_mixG}\vspace{-0.35cm}
\end{wrapfigure}
Flow matching (FM), closely related to diffusion models (DMs) \citep{sohl2015deep,ho2020denoising,song2020denoising}, has emerged as a flexible framework for generative modeling, offering a principled way to learn transport between two distributions (cf.~\cite{lipman2023flow,liu2023flow,albergo2023building}). By approximating the instantaneous velocity field of this transport with a neural network, FM enables high-fidelity multimodal generation by solving the ordinary differential equation (ODE) with the neural network-approximated vector field as its forcing term \citep{esser2024scaling,ma2024sit,polyak2024movie,jing2024alphafold,campbell2024generative}. Nevertheless, this high-fidelity generation requires multiple expensive neural network evaluations, counted by the number of function evaluations (NFEs) \citep{chen2018neural}.

Several approaches aim to accelerate diffusion- and flow-based models for high-fidelity generation with only a few NFEs. Among these, consistency models (CMs) \cite{song2023consistency,geng2024consistency,song2023improved,lu2025simplifying} achieve remarkable performance and efficiency. 
Distillation is a noticeable idea; for instance, local FM \citep{xu2024local} breaks the flow into local sub-flows, enabling smaller models and easier distillation.

Recently, flow maps \citep{boffi2024flow,boffi2025build} and mean flows (MeanFlows) (\cite{geng2023one}; cf.~Section~\ref{sec:background}) have been proposed to enable aggressive 1-NFE generation, and a prominent advantage of flow maps and MeanFlows is that they require no pre-training, distillation, or curriculum learning. Empirically, MeanFlows achieve high-quality image generation with fewer transport steps than FM models. However, preserving this quality typically requires multiple evaluations of the mean velocity field, as collapsing the process to 1-NFE often causes significant performance degradation. We showcase this issue by sampling a mixture Gaussian distribution using MeanFlow; see Section~\ref{subsec:density-estimation} for experimental details. Figure~\ref{fig:rmflow_mixG} shows the significant gap between exact (data) and sampled distributions when using 1-NFE MeanFlow, and this gap reduces as NFE increases.

\begin{wrapfigure}{l}{0.6\columnwidth}
\centering
\begin{tabular}{cc}
\hspace{-0.2cm}\includegraphics[width=0.28\columnwidth]{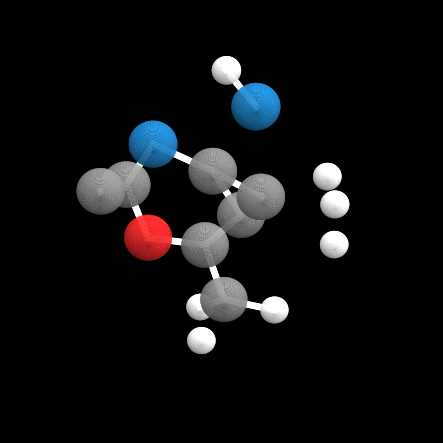}&
\hspace{-0.3cm}\includegraphics[width=0.28\columnwidth]{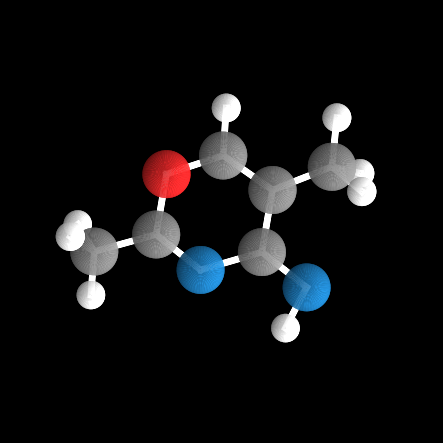}\\
{
1-NFE MFM} &{
1-NFE RMFlow ({\bf ours})}\\
\end{tabular}\vspace{-0.1cm}
\caption{
Contrasting MeanFlow with RMFlow, under the same context, for QM9 molecule generation.}\vspace{-0.15cm}
\label{fig:Contrasting_molecules}
\end{wrapfigure}
We further showcase the significant generation error of 1-NFE MeanFlow for the benchmark QM9 molecule generation \citep{ramakrishnan2014quantum}; see Section~\ref{subsec:QM9} for experimental details and additional results. Figure~\ref{fig:Contrasting_molecules} illustrates that 1-NFE MeanFlow produces an invalid structure, where the molecule is fragmented into multiple disconnected pieces. Indeed, in our experiments, we consistently observed that 1-NFE MeanFlow frequently generates invalid structures. Additional quantitative results in Section~\ref{subsec:QM9} further confirm the significant errors associated with 1-NFE MeanFlow generation.

The above numerical results motivate us to study the following problem:

\begin{center}
\emph{Can we improve the performance of 1-NFE MeanFlows for multimodal generation?}
\end{center}

\begin{figure}[!ht]
\centering
\includegraphics[width=0.95\columnwidth]{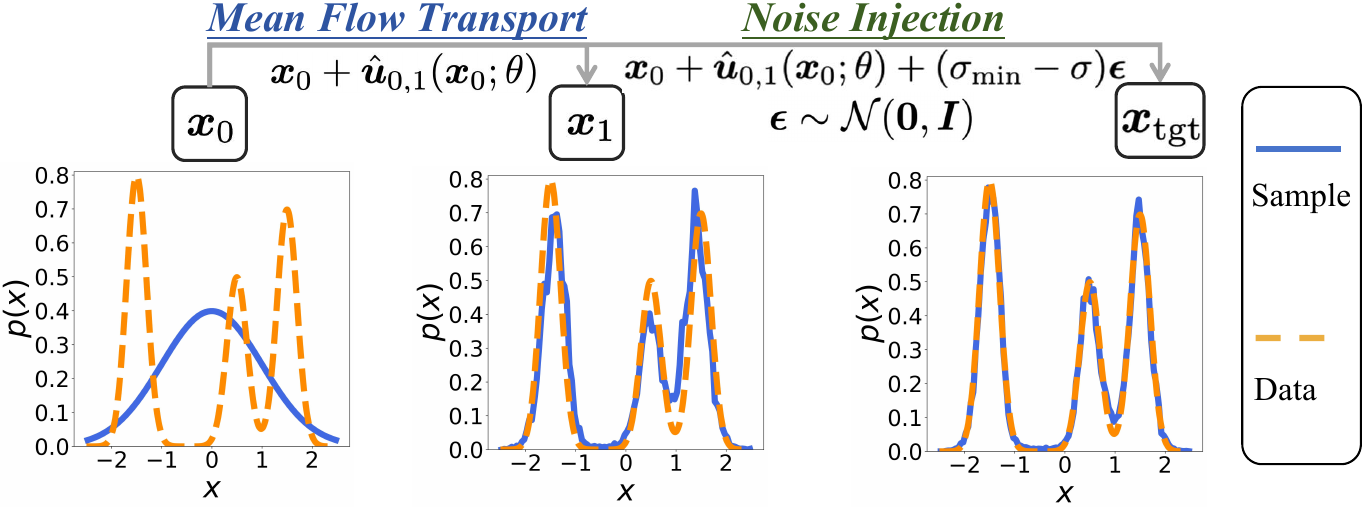}\\ \vspace{-0.2cm}
\caption{
Schematic of our proposed RMFlow: it first applies 1-NFE MeanFlow transport, then refines the result by a subsequent noise-injection step; see Section~\ref{sec:model-design}. The average velocity $\hat\vu_{0,1}(\vx_0;\theta)$ of RMFlow is trained by incorporating the maximum likelihood objective into the MeanFlow framework, as in~\eqref{eq:RMFlow-objective}.}
\label{fig:rmflow_instruction}\vspace{-0.15cm}
\end{figure}

\subsection{Our Contributions}
We propose RMFlow---an improved 1-NFE MeanFlow model for multimodal generation. RMFlow leverages 1-NFE MeanFlow for coarse transport, accompanied by a subsequent tailored noise-injection step to refine the generation; Fig.~\ref{fig:rmflow_instruction} depicts the idea of RMFlow. The results in Figs~\ref{fig:rmflow_mixG} and \ref{fig:Contrasting_molecules} demonstrate that RMFlow achieves a substantial improvement in generation quality over 1-NFE MeanFlow. In particular, it effectively mitigates invalid structures, producing coherent and valid molecular graphs. Our key contributions are:
\begin{itemize}[leftmargin=3mm]
\item We propose RMFlow to enable 1-NFE  high-fidelity multimodal generation by integrating the guidance encoding with a tailored noise-injection refinement strategy; see Section~\ref{sec:model-design}.

\item 
We design a theoretically principled training objective for RMFlow that balances minimizing the Wasserstein distance between probability paths and maximizing the likelihood of the learned target distribution; see Section~\ref{sec:Model-training}.

\item We show the compelling, often (near) state-of-the-art, results of RMFlow for benchmark text-to-image, text-to-structure, and time series generation (see Section~\ref{sec:experiments}).
\end{itemize}

\subsection{Additional Related Works}
To our knowledge, this work is the first to improve MeanFlows by introducing a noise-injection refinement for 1-NFE generation. This differs from existing couplings of flow and diffusion models, such as Diff2Flow \citep{schusterbauer2025diff2flow}, which transfers knowledge from pretrained diffusion models to flow matching models, and generator matching \citep{patel2024exploring}, which connects diffusion and flow matching under Markov generative processes.

Another line of work studies error control in FM. Prior analyses of probability flow ODEs and FM \citep{song2021maximum, lu2022maximum, lai2023fp, albergo2023stochastic} show that FM alone cannot guarantee likelihood maximization or KL divergence minimization between target and learned distributions.


\subsection{Organization}
We organize the rest of this paper as follows: We provide necessary background materials in Section~\ref{sec:background}. We present our proposed 1-NFE RMFlow in Section~\ref{sec:model-design}. We present our training loss function for RMFlow in Section~\ref{sec:Model-training}. Our numerical results for RMFlow in Section~\ref{sec:experiments}. Technical proofs and additional experimental details and results are provided in the appendix.

\section{Background}\label{sec:background}
In this section, we provide a brief review of flow-based generative models, especially MeanFlows. For a detailed exploration of FM, we refer the reader to \citep{lipman2023flow, liu2023flow, fukumizu2024flow}. 
For a given data $\vx_1 = \vx_{\rm data}\sim p$ and a prior sample $\vx_0\sim q$ (e.g., standard Gaussian $\gN({\bf 0},\mI)$), a (conditional) flow path---connecting the two samples---can be constructed as $\vx_t=a_t\vx_1+b_t\vx_0$ with $a_t$ and $b_t$ being predefined schedules. 
A common choice is $a_t = 1 - t$ and $b_t = t$, which corresponds to rectified flow \citep{liu2023flow}. 
This interpolation can be equivalently expressed as the solution to the ODE \(\dot{\vx}_t = \vu_t(\vx_t | \vz)\), where \(\vz = (\vx_0, \vx_1)\) denotes the coupling of start and end points, and \(\vu_t(\vx_t | \vz) = \dot{a}_t \vx_1 + \dot{b}_t \vx_0\) is the conditional vector field. 
FM learns an unconditional vector field \(\vu_t(\vx) := \mathbb{E}_{\vz} \left[ \vu_t(\vx | \vz)| \vx_t = \vx \right]\), 
which does not require knowledge of the pair \(\vz = (\vx_0, \vx_1)\). This is achieved by training a neural network \(\hat{\vu}_t(\vx; \theta)\) to minimize the objective:
\begin{equation}
\label{eq:CFM}
\mathcal{L}_{\text{CFM}}(\theta) := \mathbb{E}_{t, \vz} \left[ \left\| \hat{\vu}_t(\vx_t; \theta) - \vu_t(\vx_t | \vz) \right\|^2 \right].
\end{equation}
After training, we generate data by integrating $\frac{d\bm \vx_t}{dt} = \hat\vu_t(\vx_t;\theta)$ from $t=0$ to $1$, with $\vx_0\sim q$. 

Although FM is conceptually simple, sample generation requires multiple evaluations of \(\hat{\vu}_t(\vx_t; \theta)\), which can be computationally intensive. To address this inefficiency issue, MeanFlow learns an averaged velocity field based on the instantaneous velocity field \(\vu_t(\vx_t)\), defined as:
\begin{equation}\label{eq:mean-vector-field}
\vu_{t,r}(\vx_t) := \frac{\vx_r-\vx_t}{r-t} = \frac{1}{r-t}\int_t^r\vu_{s}(\vx_s)ds.
\end{equation}
This allows data generation by transporting \(\vx_t\) to \(\vx_r\) using the approximate average velocity field \(\hat{\vu}_{t,r}\):
 \begin{equation}\label{eq:MFM-generation}
\vx_r=\vx_t+(r-t)\hat{\vu}_{t,r}(\vx_t;\theta).
\end{equation}
In particular, 1-NFE generation corresponds to $\vx_1=\vx_0+\hat\vu_{0,1}(\vx_0;\theta)$. For Multi-NFE generation, $\hat\vu_{t,r}(\vx;\theta)$ is evaluated sequentially on a chosen grid $0=\tau_0<\cdots<\tau_n=1$ and is applied between consecutive grid points to transport samples. This approach achieves high-fidelity generation with significantly fewer NFEs compared to FM models that rely on the instantaneous velocity field \(\vu_t(\vx)\).

In MeanFlows, the mean velocity field $\vu_{t,r}(\vx)$ is approximated by a neural network $\hat{\vu}_{t,r}(\vx;\theta)$, with the weights $\theta$ being calibrated by minimizing the following conditional mean flow matching $\gL_{\rm CMFM}$ loss function:
\begin{equation}\label{eq:mean-flow-matching}
\begin{aligned}
&\gL_{\rm CMFM}(\theta) := \mathbb{E}_{t,r,\vz}\big[ \|\hat{\vu}_{t,r}(\vx;\theta)-\texttt{sg}\left(\vu^{\rm tgt}_{t,r}(\vx;\theta)\right) \|^2 \big],
\end{aligned}
\end{equation}
where $0\leq t\leq r\leq 1$ are uniform samples from the interval $[0,1]$ and $\vu^{\rm tgt}_{t,r}$ is the target defined as:
$$
\vu_{t,r}^{\rm tgt}(\vx;\theta) := \vu_t(\vx|\vz) + (r-t)\big[ \nabla\hat{\vu}_{t,r}(\vx;\theta)\cdot\vu_t(\vx|\vz) + \partial_t\hat{\vu}_{t,r}(\vx;\theta) \big],
$$
with $\texttt{sg}$ denoting a stop-gradient operation. This stop-gradient approach prevents higher-order optimization while ensuring that zero loss guarantees dynamical consistency. The target velocity \(\vu_{t,r}^{\rm tgt}\) is efficiently computed using Jacobian-vector products (\texttt{jvp}) in autodiff libraries such as \texttt{PyTorch} \citep{paszke2019pytorch} or \texttt{JAX} \citep{bradbury2018jax}.

\section{The Design of RMFlow}\label{sec:model-design}
In this section, we describe the design of 1-NFE RMFlow for high-fidelity generation, with or without multimodal guidance.

\subsection{Motivation}
In practice, the true data distribution $\vx_{\rm data}$ is unavailable due to its complexity and the limited nature of observed data. Following the standard practice in the field, as established in works such as \citep{ho2020denoising, song2021maximum, lu2022maximum, lipman2023flow}, we approximate it with a noisy, smoothed version $\vx_{\rm tgt} = \vx_{\rm data} + \sigma_{\min} \boldsymbol{\epsilon}\sim p_{\rm tgt}$, where $\boldsymbol{\epsilon} \sim \mathcal{N}(\mathbf{0}, \mI)$ and $\sigma_{\min}$ is small (e.g., $10^{-3}$). This approach ensures stability and robust learning of the data distribution.

MeanFlows learn a neural network, by minimizing $\gL_{\rm CMFM}$ in \eqref{eq:mean-flow-matching}, to transport a prior sample $\vx_0$ directly to the noisy target, i.e., $\vx_1 = \vx_{\rm tgt}$. \cite{boffi2024flow, boffi2025build} showed that this approach reduces the Wasserstein distance between the target distribution $p_{\rm tgt}$ and the learned distribution $p_\theta$:
\begin{restatable}{theorem}{thwasser}\label{th:wasser}[\cite{boffi2025build}]
There exists a constant $M > 0$ such that:
\begin{equation}
    M \cdot \gL_{\rm CMFM}(\theta) \geq W^2_2(p_{\rm tgt}, p_\theta):= \inf_{\gamma \in \Pi(p_{\rm tgt}, p_\theta)} \mathbb{E}_{(x,y)\sim \gamma} \big[\|x - y\|^2\big],  
\end{equation}
where $M$ is a constant, $ W^2_2(p_{\rm tgt}, p_\theta)$ denotes the Wasserstein distance between $p_{\rm tgt}$ and $p_\theta$, and $\Pi(p_{\rm tgt}, p_\theta)$ is the set of all joint distributions with marginals $p_{\rm tgt}$ and $p_\theta$.
\end{restatable}

While controlling the Wasserstein distance provides a meaningful measure of distributional alignment, empirical evidence indicates that FM enforces additional constraints, such as 
KL divergence \citep{lu2022maximum}, often achieves superior generative performance over the FM baseline. With this in mind, we aim to enhance the fidelity of 1-NFE MeanFlows, pushing beyond current limitations in a manner analogous to improvements seen in FMs.

\subsection{Noise Injection Refinement}
We decompose the generation process into two stages. In the first stage, a 1-NFE MeanFlow transports the prior $\vx_0$ to an intermediate noisy state
\begin{equation}
\vx_1 = 
\vx_{\rm data}
+ \sigma \epsilon_1, \quad \text{with } \epsilon_1 \sim \mathcal{N}(\mathbf{0}, \mI) \text{ and } \sigma < \sigma_{\min}. 
\label{eq:inter_noisy_state}
\end{equation}
In the second stage, a \emph{single noise injection step} is applied:
\begin{equation}
\vx_{\rm tgt} = \vx_1 + \sqrt{\sigma_{\min}^2 - \sigma^2}\cdot\epsilon_2, \quad \epsilon_2 \sim \mathcal{N}(\mathbf{0}, \mI),
\label{eq:inj_noisy_state}
\end{equation}
to generate the final sample. This additional noise injection aligns with the designs of VAEs~\citep{kingma2013auto}, allowing principled likelihood maximization via a loss term derived from the evidence lower bound (ELBO) \citep{wainwright2008graphical} to optimize the MeanFlow parameters. We will prove in Theorem~\ref{thm:KL} that this formulation enables control over the KL divergence between the target distribution $p_{\rm tgt}$ and the learned distribution $p_\theta$.

In summary, {\bf our data generation process} is defined as
\begin{equation}\label{eq:RMFlow-generation}
\boxed{
\hat{\vx}_{\rm tgt}= \vx_0 + \hat{\vu}_{0,1}(\vx_0;\theta) + \sqrt{\sigma_{\rm min}^2-\sigma^2}\cdot\epsilon_2, \quad \epsilon_2 \sim \mathcal{N}(\mathbf{0}, \mI),
}
\end{equation}
where $\hat{\vu}_{0,1}(\vx_0;\theta)$ denotes the learned average velocity field. Although RMFlow is conceptually a two-stage framework, \eqref{eq:RMFlow-generation} demonstrates that generation is performed in a single step: the learned flow is evaluated once (1-NFE), and a noise term is added in parallel to produce the output.

\subsection{Multimodality}
To support cross-modality generation, we incorporate an encoder $\phi_\omega(\vc)$ that embeds conditioning signals (e.g., text prompts). The prior samples for both guided (potentially multimodal) and unguided generation are defined as
\[
\vx_0 = 
\begin{cases}
\phi_\omega(\vc) + \sigma_c \epsilon, & \text{for guided generation},\\[1ex]
\epsilon, & \text{for unguided generation},
\end{cases}
\]
where $\epsilon \sim \mathcal{N}(\mathbf{0}, \mI)$. This design allows the flow to incorporate multimodal guidance if available, while defaulting to unconditional generation otherwise. Here, $\phi_\omega(\cdot)$ is an encoder chosen following common practice (see Section~\ref{sec:experiments}), and $\sigma_c \ll 1$ (e.g., $10^{-3}$) is pre-chosen to control perturbations.  

Specifically, for a given data pair $(\vx_{\rm data}, \vc)$, we train the MeanFlow to transport the prior sample $\vx_0 = \phi_\omega(\vc) + \sigma_c \epsilon$ to the intermediate target $\vx_1 = \vx_{\rm data} + \sigma \epsilon_1$, where $\epsilon_1 \sim \mathcal{N}(\mathbf{0}, \mI)$. 
{\it The encoder and MeanFlow are optimized jointly}, and we will discuss the training objective in Section~\ref{sec:Model-training}.

\begin{remark}
Our proposed RMFlow differs from MeanFlow in two aspects:
(1) We apply a tailored encoder to the guidance,  (2) we add a noise injection step to refine the generation result.
\end{remark}

\section{The Training of RMFlow}\label{sec:Model-training}
In this section, we present the training procedure for RMFlow.
We first establish the theoretical foundation of noise-injection refinement, showing that it enables likelihood maximization of the learned distribution with respect to the target distribution. Building on this, we introduce a joint training objective that combines $\gL_{\rm CMFM}$ (for Wasserstein control) with likelihood maximization and optional guidance regularization, ensuring both fidelity and flexibility in guided generation. Finally, we adopt parameter-efficient fine-tuning (PEFT; cf. \citep{hu2022lora,dettmers2023qlora}) to implement RMFlow for large-scale tasks.

\subsection{Likelihood Maximization}\label{subsubsec:likelihood-maximization}
In this section, we show that the noise-injection step in \eqref{eq:RMFlow-generation} enables likelihood maximization during RMFlow training. Specifically, for a given prior sample $\vx_0$, the intermediate sample generated by the MeanFlow is
\[
\vx_1 = \vx_0 + \hat{\vu}_{0,1}(\vx_0;\theta).
\]  
By \eqref{eq:RMFlow-generation}, the conditional distribution of the final generated sample given the prior is
\[
\hat\vx_{\rm tgt} \mid \vx_0 \sim \mathcal{N}\Big( \vx_0 + \hat{\vu}_{0,1}(\vx_0; \theta), (\sigma_{\min}^2-\sigma^2) \mI \Big).
\]
This specifies a parametric conditional distribution. Given an observed target $\vx_{\rm tgt}$, the corresponding conditional log-likelihood is
\begin{equation}\label{eq:log-likelihood}
\log p_\theta(\vx_{\rm tgt} \mid \vx_0) = -\frac{1}{2 ((\sigma_{\min}^2-\sigma^2))} \big\| \vx_{\rm tgt}  - \big( \vx_0 + \hat{\vu}_{0,1}(\vx_0; \theta) \big) \big\|^2 + C,
\end{equation}
where $C = - \frac{d}{2} \log (2 \pi (\sigma_{\min}^2-\sigma^2))$ and $d$ is the dimensionality of the data. Notice that~\eqref{eq:inter_noisy_state} and~\eqref{eq:inj_noisy_state} indicate the observed target $\vx_{\rm tgt}=\vx_{\rm data} + \sigma_{\rm min}\epsilon$, where $\epsilon\sim\gN(\bm{0},\mI)$. 

Therefore, we define the following loss term to maximize the likelihood:
\begin{equation}
\gL_{\rm NLL} := \mathbb{E}_{\vx_0, \vx_{\rm data}, \epsilon}\Big[ \big\| (\vx_{\rm data} + \sigma_{\min}\epsilon) - (\vx_0 + \hat{\vu}_{0,1}(\vx_0; \theta)) \big\|^2 \Big].
\end{equation}

The following theorem formalizes the theoretical guarantee of the noise-injection refinement. In particular, it demonstrates that minimizing the loss $\mathcal{L}_{\rm NLL}$ maximizes the expected log-likelihood, thereby reducing the KL divergence between the target and learned distributions.
\begin{restatable}{theorem}{thmKL}\label{thm:KL}
The negative log-likelihood loss $\mathcal{L}_{\rm NLL}$ provides a lower bound on the expected log-likelihood of the target distribution:
\begin{equation}
-A\cdot \mathcal{L}_{\rm NLL}  + C \leq \mathbb{E}_{\vx_{\rm tgt}}[\log p_\theta(\vx_{\rm tgt})] = -H(p_{\rm tgt}) - D_{\rm KL}(p_{\rm tgt} \,\|\, p_\theta),
\end{equation}
where $H(p_{\rm tgt}):=  -\mathbb{E}_{\vx_{\rm tgt}}[\log p_{\rm tgt}] $ is the entropy of $p_{\rm tgt}$, $D_{\rm KL}(p_{\rm tgt}|| p_\theta):=  \mathbb{E}_{\vx_{\rm tgt}}[\log \frac{p_{\rm tgt}}{p_\theta}]$ denotes the KL divergence between the target and the learned distributions, and $A, B>0$ are constants.\end{restatable}


\subsubsection{Joint Training Objective}
\label{sec:joint-objective}
RMFlow is trained by jointly optimizing the original MeanFlow loss (Wasserstein control) and likelihood maximization, resulting in the following objective function:
\begin{equation}\label{eq:RMFlow-objective}
\begin{aligned}
\gL_{\rm RMFlow}(\theta,\omega) &= \underbrace{\gL_{\rm CMFM}}_{\text I} +\underbrace{\lambda_1 \gL_{\rm NLL}}_{\text{II}}
+ \underbrace{\lambda_2 \mathbb{E}_{(\vx_{\rm data},\vc)}[\|\phi_\omega(\vc)\|^2] }_{\text{III}},
\end{aligned}
\end{equation}
where $\lambda_1,\lambda_2\geq 0$ are two hyperparameters. We remark that Term I controls the gap between the probability flows of the exact and approximated mean velocities in intermediate states, Term II  for likelihood maximization, and Term III is designed for guided generation and is set to 0 for unguided generation. Here, the expectation in III is taken over all data-guidance pairs $(\vx_{\rm data},\vc)$.

\begin{remark}
Term III in \eqref{eq:RMFlow-objective} can be considered as a regularization on the prior distribution, and a similar term is used in training VAE \citep{kingma2013auto}. Empirically, we observe that term III can be very large, resulting in substantial performance degradation.
\end{remark}

\subsection{Memory-Efficient Fine-Tuning}\label{subsec:memory-efficient-fine-tuning}
For relatively small-scale tasks, we train our RMFlow by directly minimizing $\gL_{\rm RMFlow}$. Compared to $\gL_{\rm CMFM}$, our new objective $\gL_{\rm RMFlow}$ introduces additional gradient pathways, increasing memory footprint. To balance efficiency and performance for large-scale tasks, we first train the MeanFlow model by minimizing $\gL_{\rm CMFM}$, and then fine-tune it using PEFT \citep{hu2022lora,dettmers2023qlora}, with $\gL_{\rm RMFlow}$ as a supervised objective in our large-scale experiments on text-to-image and molecule generation tasks. During fine-tuning, we further strengthen training by integrating 1-NFE sampling with a policy-gradient objective that incorporates physical feedback on sample quality for molecule generation tasks, as described in~\cite {zhou2025guiding}.

\section{Numerical Experiments}\label{sec:experiments}
In this section, we validate the efficacy and efficiency of RMFlow for both guided and unguided sample generation. We consider two synthetic tasks: sampling a 1D mixture Gaussian distribution and a 2D checkerboard density (Section~\ref{subsec:density-estimation}). We also consider several benchmark tasks, including context-to-molecular structure generation (Section~\ref{subsec:QM9}), sampling trajectories of dynamical systems (time series; Section~\ref{subsec:time-series}), and text-to-image generation (Section~\ref{subsec:text2image}).

\textbf{Software and Equipment.} 
We implement synthetic tasks, context-to-molecule generation, and text-to-image generation using \texttt{PyTorch}. We implement the time series generation task using \texttt{JAX}. Additionally, we use \texttt{Torch DDP} and \texttt{torch.compile} to optimize the model execution for context-to-molecule and text-to-image generation. All the experiments are carried out on multiple NVIDIA RTX 3090/4090 GPUs.

\textbf{Training Setups.} See Appendix~\ref{sec:experiment_setup} for the details of training setups.

\textbf{Evaluation Metrics:} For synthetic tasks and time-series generation, we evaluate performance using the estimated KL divergence and total variation (TV) distance between the generated samples and the ground-truth. Both KL and TV are computed from densities obtained via histogram-based estimation of the sample and ground-truth distributions. For molecule generation, we predict bond types from pairwise interatomic distances and atom types, and then compute atom and molecule stability, following~\cite{hoogeboom2022equivariant}. For the image generation task, we assess sample quality using the Fréchet Inception Distance (FID)~\citep{heusel2017gans}. 

We use NFE to measure generation efficiency following \citep{geng2025mean}. Notice that the Gaussian noise injection step takes negligible time compared to the neural network function evaluation.

\begin{table}[!ht]
\centering
\fontsize{8.5}{8.5}\selectfont
\begin{tabular}{lccc|c}
\toprule
&1-NFE MeanFlow& 8-NFE MeanFlow& 32-NFE MeanFlow 
& 1-NFE RMFlow ({\bf ours})\\
\midrule
TV & 1.4422 & 0.7977 & 0.6737 
& 0.7567  \\
KL & 0.8074 & 0.4074 & 0.1017 
& 0.2332 \\
\bottomrule
\end{tabular}
\vspace{-0.5em}
\caption{
Contrasting 1-NFE RMFlow with 1/8/32-NFE MeanFlow for mixture Gaussian sampling. 1-NFE RMFlow outperforms both 1- and 8-NFE MeanFlows, while slightly worse than 32-NFE MeanFlow.
}
\label{tb:mixG_tb}
\end{table}

\begin{table}[!ht]
\centering
{\fontsize{8.5}{8.5}\selectfont
\begin{tabular}{lccc|c}
\toprule
&\includegraphics[width=0.15\linewidth]{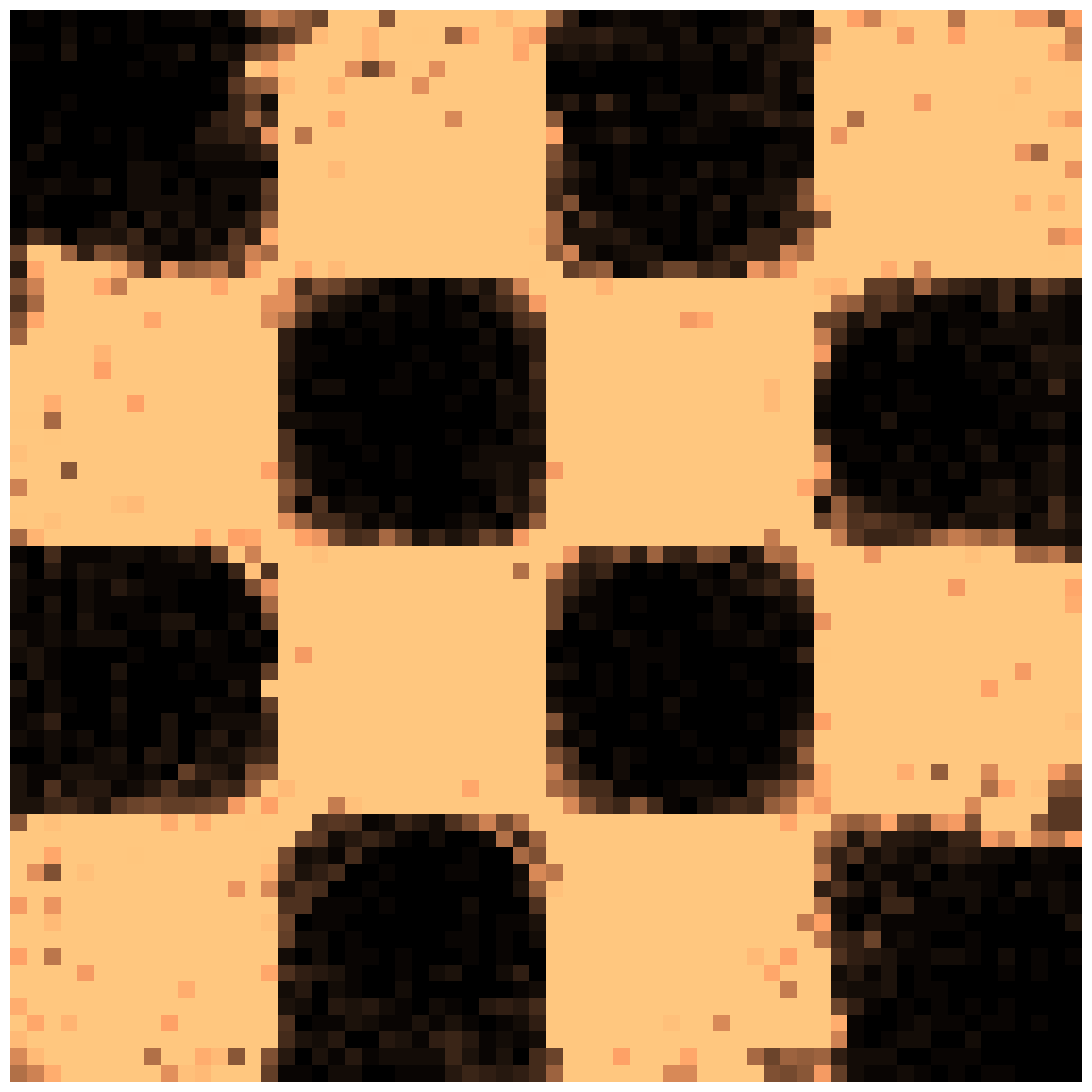}&
\includegraphics[width=0.15\linewidth]{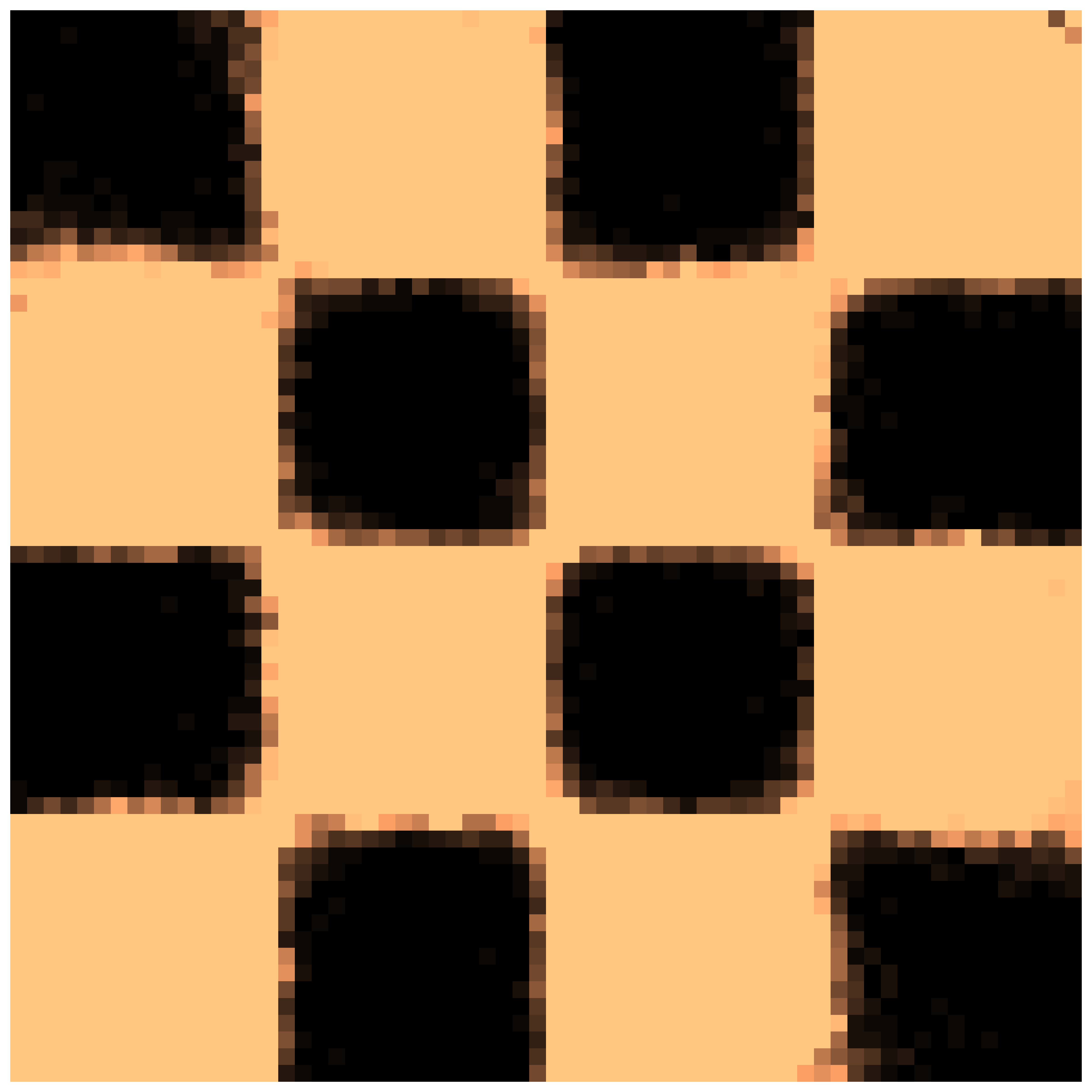}&
\includegraphics[width=0.15\linewidth]{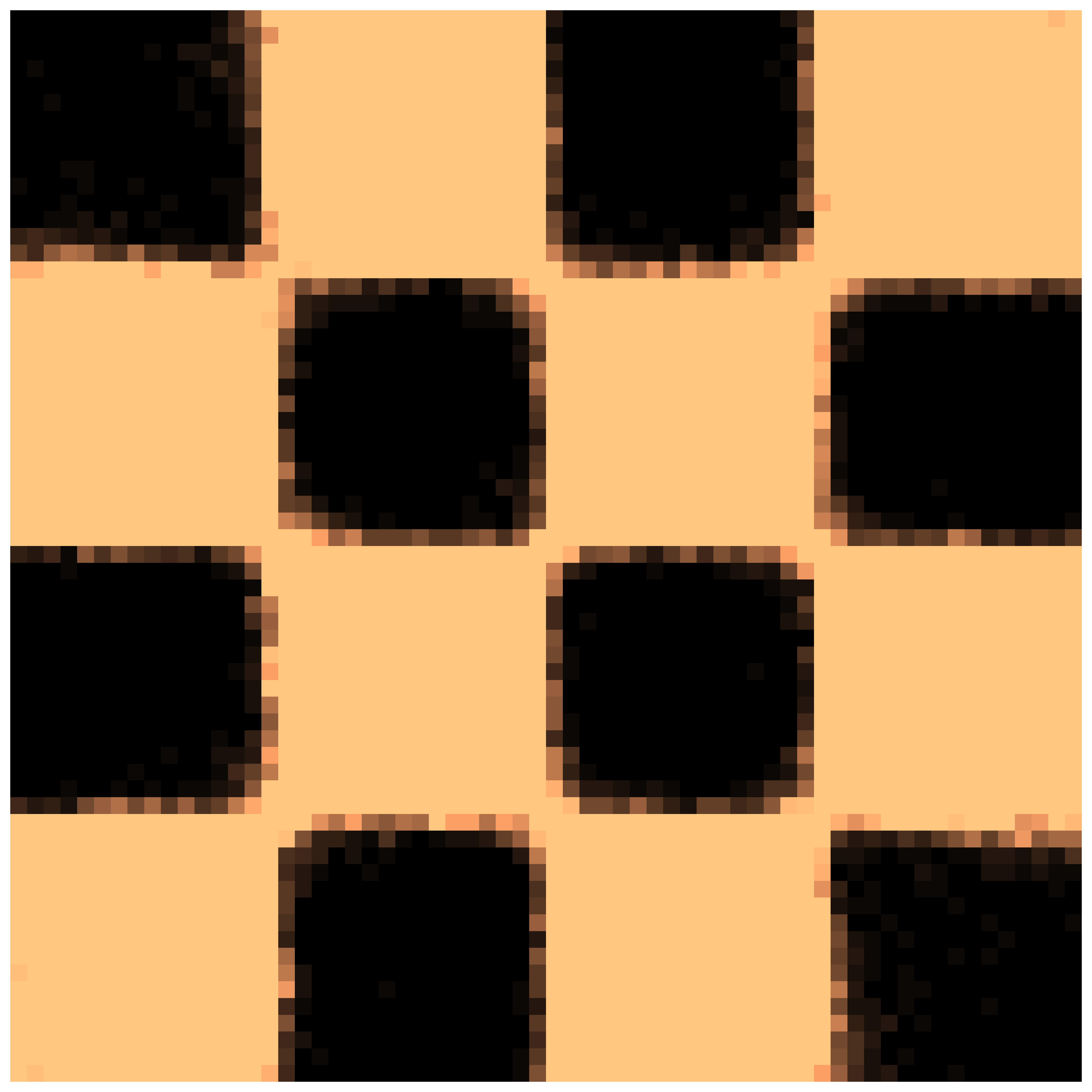}&
\includegraphics[width=0.15\linewidth]{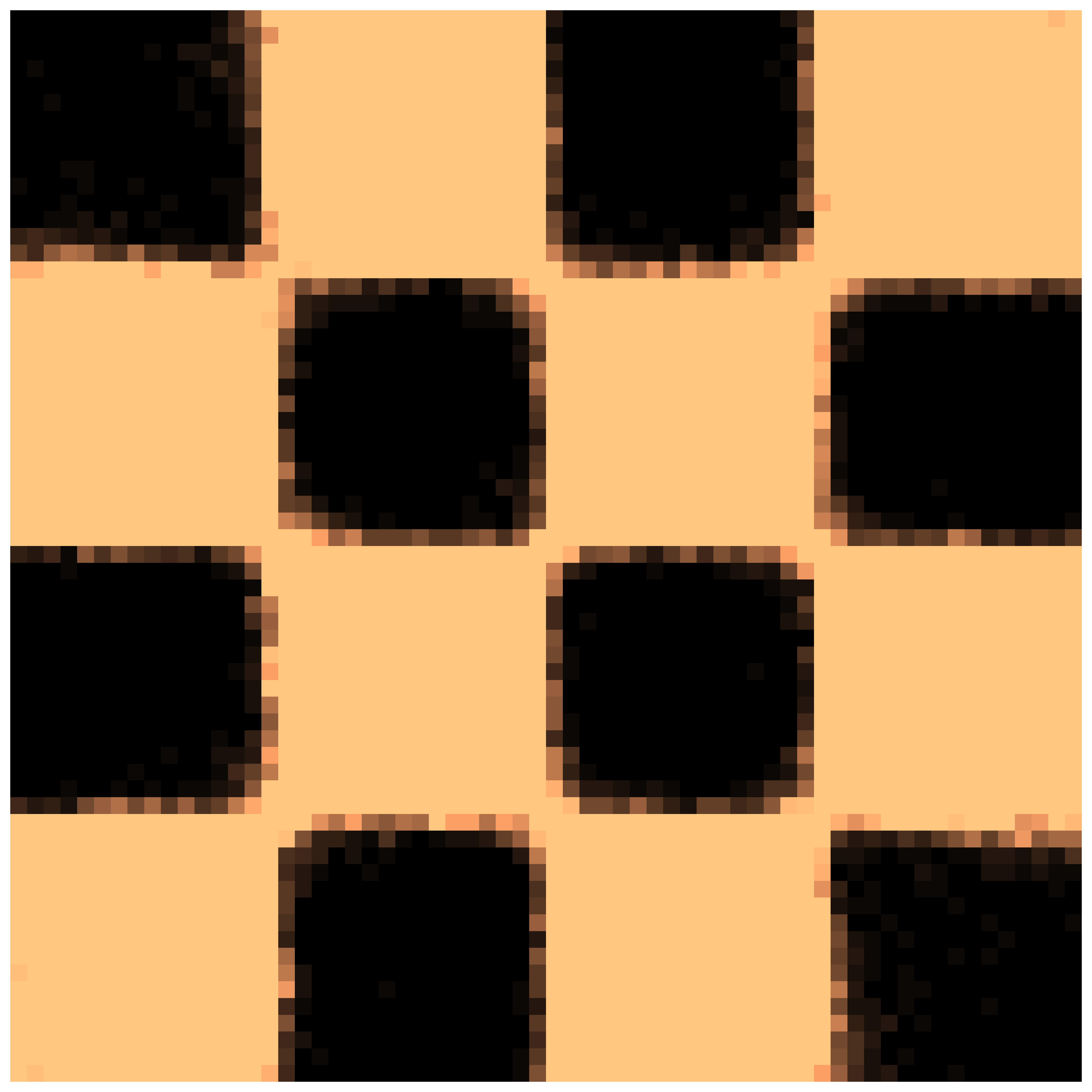}
\\
\midrule
&1-NFE MeanFlow& 8-NFE MeanFlow& 32-NFE MeanFlow
& 1-NFE RMFlow ({\bf ours})\\
\midrule
TV & 0.238& 0.167& 0.155
& 0.173\\
KL & 0.311& 0.139& 0.118
& 0.163\\
\bottomrule
\end{tabular}}
\vspace{-0.3em}
\caption{
Contrasting 1-NFE RMFlow with 1/8/32-NFE MeanFlow for checkerboard density sampling. 1-NFE RMFlow significantly outperforms 1-NFE MeanFlow, closing the performance gap to multi-NFE MeanFlow.
}
\label{tb:checkerboard_tb}
\end{table}

\subsection{Synthetic Tasks
}\label{subsec:density-estimation}
In this experiment, we train a simple ResNet-based model under both MeanFlow and RMFlow frameworks for $10^5$ iterations using a batch size of $256$ to sample (1) 1D Gaussian mixture $p_{\rm data} = 0.35\gN(1.5, 0.04)+0.25\gN(0.5, 0.04)+0.4\gN(-1.5, 0.04)$, and (2) 2D checkerboard where the probability density resembles a checkerboard pattern. 
We consider 1/8/32-NFE MeanFlow and 1-NFE RMFlow for sample generation.
Tables \ref{tb:mixG_tb} and \ref{tb:checkerboard_tb} show that 1-NFE RMFlow significantly outperforms 1-NFE MeanFlow, closing the performance gap to multi-NFE MeanFlow.



\subsection{Context-to-Molecule: QM9 Generation}\label{subsec:QM9}
We train MeanFlow and RMFlow for context-to-molecule generation on the QM9 dataset~\cite{ramakrishnan2014quantum}, a benchmark containing atomic coordinates and quantum-chemical properties for 130k small molecules with up to 9 heavy atoms (up to 29 atoms including hydrogens). Following~\cite{hoogeboom2022equivariant}, we perform condition generation on seven molecular properties: (1) number of atoms, (2) HOMO, (3) LUMO, (4) $\alpha$ (isotropic polarizability), (5) gap, (6) $\mu$ (dipole moment), and (7) $C_v$ (heat capacity). These properties are concatenated into a context vector and mapped to the data space using $\phi_{\omega}(\vc)$, parameterized by a single EGNN block~\cite{garcia2021n}.

Our model backbone follows the EGNN architecture in~\citep{garcia2021n,hoogeboom2022equivariant}, augmented with a time-embedding module for the additional scalar time variable $r$. In addition, molecule stability is used as the reward within the RL framework, following the approach of~\cite{zhou2025guiding},
to provide 
feedback during training (see Section~\ref{subsec:memory-efficient-fine-tuning} and Appendix~\ref{sec:rlpf_loss}). We adopt the train/val/test splits of~\cite{anderson2019cormorant}, comprising $100k/18k/13k$ molecules, respectively. Table~\ref{tb:qm9_results} shows that 1-NFE RMFlow attains state-of-the-art performance, whereas competing SOTA methods require $n$-NFE with $n\gg 1$.
Figure~\ref{fig:qm9} depicts a few randomly generated molecules and the corresponding contexts.

\begin{figure}[!ht]
\centering
\begin{tabular}{cccccc}
\includegraphics[width=0.13\linewidth]{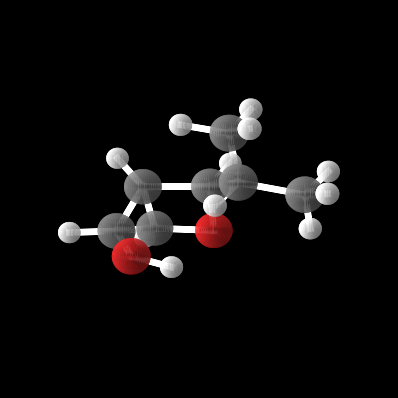}&
\includegraphics[width=0.13\linewidth]{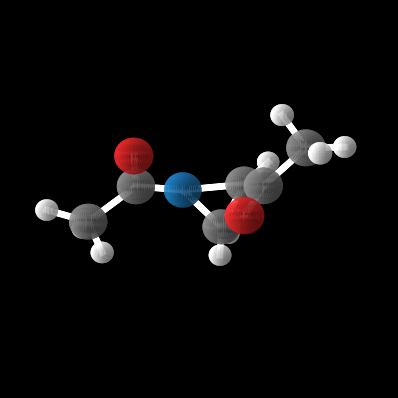}&
\includegraphics[width=0.13\linewidth]{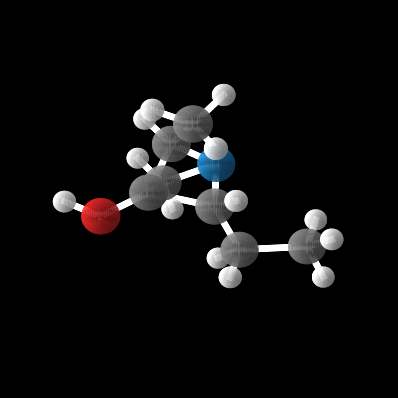}&
\includegraphics[width=0.13\linewidth]{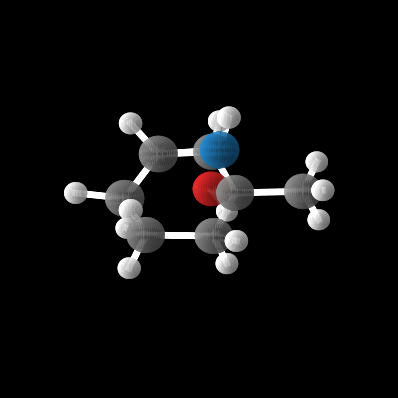}&
\includegraphics[width=0.13\linewidth]{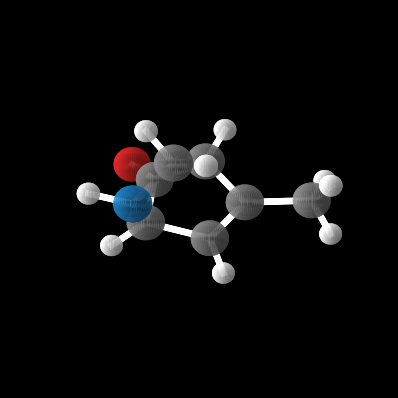}&
\includegraphics[width=0.13\linewidth]{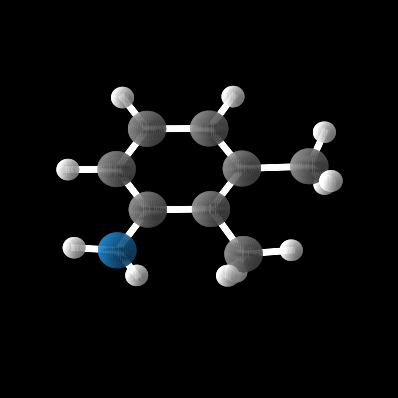}\\
(1) & (2) &(3) &(4) &(5) &(6) \\
\includegraphics[width=0.13\linewidth]{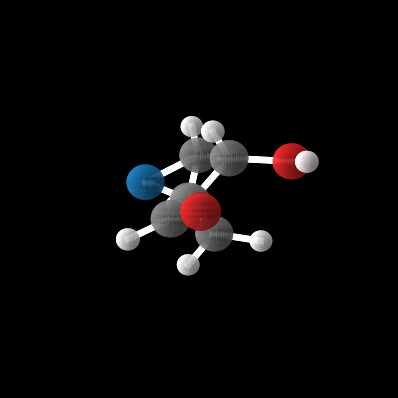}&
\includegraphics[width=0.13\linewidth]{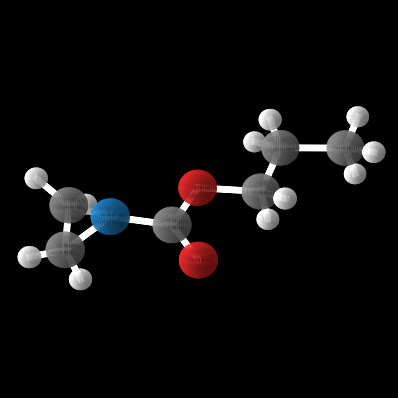}&
\includegraphics[width=0.13\linewidth]{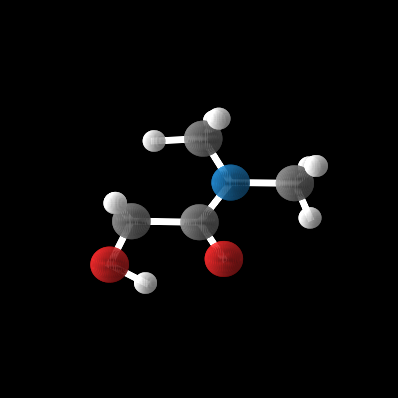}&
\includegraphics[width=0.13\linewidth]{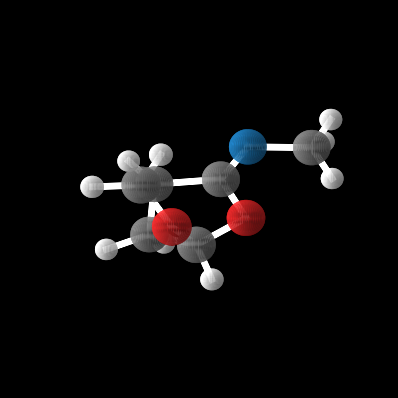}&
\includegraphics[width=0.13\linewidth]{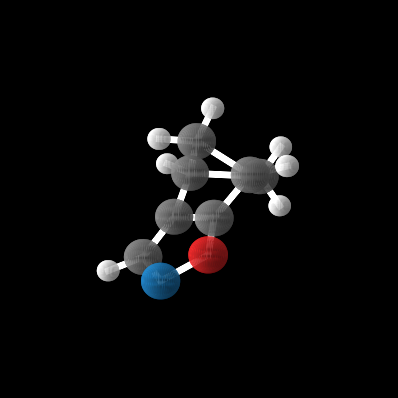}&
\includegraphics[width=0.13\linewidth]{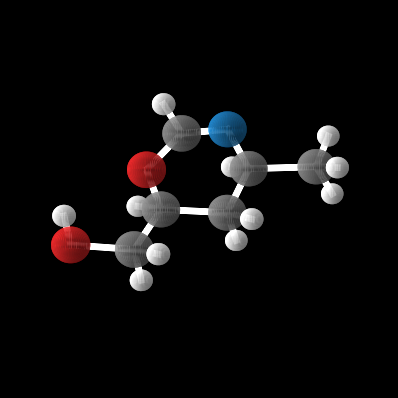}\\
(7) & (8) &(9) &(10) &(11) &(12) 
\end{tabular}
{\fontsize{8.5}{8.5}\selectfont
\begin{tabular}{l}
\textbf{Contexts:}\\
\text{(1) N atom: 21 | homo: -6.34 | lumo: 1.82 | alpha: 77.77 | gap: 8.16 | mu: 1.55 | Cv: -27.27}\\
\text{(2) N atom: 18 | homo: -6.62 | lumo: -0.85 | alpha: 74.21 | gap: 5.77 | mu: 3.27 | Cv: -19.37}\\
\text{(3) N atom: 22 | homo: -6.44 | lumo: 1.99 | alpha: 83.22 | gap: 8.43 | mu: 1.52 | Cv: -29.10}\\
\text{(4) N atom: 22 | homo: -6.56 | lumo: 2.01 | alpha: 81.16 | gap: 8.57 | mu: 0.31 | Cv: -30.52}\\
\text{(5) N atom: 18 | homo: -5.79 | lumo: -0.02 | alpha: 77.48 | gap: 5.77 | mu: 2.00 | Cv: -23.75}\\
\text{(6) N atom: 20 | homo: -5.30 | lumo: 0.34 | alpha: 91.41 | gap: 5.64 | mu: 1.54 | Cv: -24.78}\\
\text{(7) N atom: 16 | homo: -7.05 | lumo: -1.20 | alpha: 68.35 | gap: 5.84 | mu: 1.37 | Cv: -18.86}\\
\text{(8) N atom: 20 | homo: -6.94 | lumo: 0.94 | alpha: 78.01 | gap: 7.88 | mu: 2.36 | Cv: -25.09}\\
\text{(9) N atom: 16 | homo: -6.74 | lumo: 0.60 | alpha: 58.89 | gap: 7.34 | mu: 4.43 | Cv: -18.93}\\
\text{(10) N atom: 18 | homo: -6.76 | lumo: 0.77 | alpha: 73.96 | gap: 7.53 | mu: 1.95 | Cv: -24.84}\\
\text{(11) N atom: 16 | homo: -6.09 | lumo: -0.10 | alpha: 74.33 | gap: 5.98 | mu: 3.87 | Cv: -22.98}\\
\text{(12) N atom: 20 | homo: -6.82 | lumo: 0.53 | alpha: 76.55 | gap: 7.36 | mu: 1.12 | Cv: -26.04}
\end{tabular}}
\vspace{-0.3em}
\caption{
A few randomly selected RMFlow-generated molecules, together with the corresponding contexts.
}
\label{fig:qm9}
\end{figure}

\begin{table*}[!ht]
\vspace{-.25cm}
\fontsize{8.5}{8.5}\selectfont
\centering
\begin{tabular}{lcccc}
\specialrule{1.2pt}{1pt}{1pt}
\qquad Metrics & No Tuncation Error & Atomic Stab. ($\uparrow$) & Mol Stab. ($\uparrow$) & NFE($\downarrow$)\\
&(Discretization)&&&\\
\hline
{ENF} & \ding{55} & {$85_{\pm0.1}$\%} & {$4.9_{\pm0.2}$\%} & $\gg 1$ \\
E-DM \cite{hoogeboom2022equivariant}     &  \ding{55}  & $98.73_{\pm 0.1}$\% & $82.11_{\pm 0.4}$\%  & $\gg 1$    \\ 
Bridge \cite{wu2022diffusion}
& \ding{55} & $98.7_{\pm0.1}$\% & $81.8_{\pm0.2}$\% & $\gg 1$ \\
Bridge + Force \cite{wu2022diffusion} 
& \ding{55} & $98.8_{\pm0.1}$\% & $84.6_{\pm0.3}$\% & $\gg 1$ \\
GeoLDM \cite{xu2023geometric}  &  \ding{55} & $98.73$\% & $89.40_{\pm 0.5}$\% & $\gg 1$\\
GeoBFN~\cite{song2024unified} & \ding{55} & $99.0$\% & $93.9$\% & $\gg 1$ \\
E-DM + RLPF~\cite{zhou2025guiding} & \ding{55} & $99.1$\% & $93.4$\% & $\gg 1$ \\
\hline
MeanFlow w/o contexts
& \ding{51} & $98.2_{\pm0.07}$\% & $79.3_{\pm0.8}$\% & $1$ \\
MeanFlow w/ contexts
& \ding{51} & $98.4_{\pm0.05}$\% & $84.3_{\pm0.5}$\% & $1$ \\
RMFlow w/o contexts ({\bf ours})
& \ding{51} & $98.8_{\pm0.05}$\% & $90.1_{\pm0.5}$\% & $1$ \\
RMFlow w/ contexts ({\bf ours})
& \ding{51} & $98.9_{\pm0.05}$\% & $93.2_{\pm0.4}$\% & $1$ \\
RMFlow w/ contexts + RLPF ({\bf ours})
& \ding{51} & $98.9_{\pm0.05}$\% & $93.5_{\pm0.3}$\% & $1$ \\
\hline
{Data} & {--} & {$99\%$} & {$95.2\%$}   & {--} \\
\specialrule{1.2pt}{1pt}{1pt}
\end{tabular}
\vspace{-0.5em}
\caption{
Contrasting the performance of different models for QM9 molecule generation. We run RMFlow with contexts by randomly selecting $10^4$ contexts in the test dataset of QM9 five times.
} 
\label{tb:qm9_results}\vspace{-0.3cm}
\end{table*}

\subsection{Time Series: Dynamical System}\label{subsec:time-series}
Sampling trajectories in dynamical systems under event guidance is a key challenge for predicting and understanding complex phenomena such as climate and extreme events~\citep{perkins2013measurement,mosavi2018flood}. Recent works~\citep{finzi2023user} and \citep{huang2025improving} have introduced diffusion and FM models specifically designed for event-guided sampling. 

In this experiment, we perform dynamical system trajectory forecasting with MeanFlow and RMFlow, formulating it as a time series problem by discretizing the time variable $t$ into uniform intervals. Each trajectory (either from a dataset or sampled) is a discrete time series of vectors concatenated into $\vx_{\rm data}= [\vx(\tau_m)]_{m=1}^M\in\sR^{Md}$, where $M$ is the total number of time steps, $d$ is the dimension of the system, and $\vx(\tau_m)\in\sR^d$ denotes the discretized trajectory at time $\tau_m$. Our goal is to generate $\vx_{\rm data} = [\vx(\tau_m)]_{m=1}^M\in\sR^{Md}$ with 1-NFE using MeanFlow and RMFlow.

We train our models on the Lorenz and FitzHugh–Nagumo dynamical systems (see \cite[Appendix B.1]{huang2025improving} for a brief review of these two models);
using a U-Net backbone. For event guidance, where events are defined by a constraint function $E=\{\vx_{\rm data}\, |\, C(\vx_{\rm data}) > 0\}$, 
we adopt a simple but effective design: the event-guidance vector and the first three states $\vx(\tau_1), \vx(\tau_2), \vx(\tau_3)$ are embedded through an MLP $\phi_{\omega}$ into the target data space $\mathbb{R}^{Md}$. This avoids reliance on Tweedie’s formula as used in~\citep{finzi2023user,huang2025improving}.
Tables~\ref{tb:dynamic_tv} and~\ref{tb:dynamic_kl} show that RMFlow yields significantly better 1-NFE generation than MeanFlow, while achieving accuracy comparable to multi-NFE methods.

\begin{table}[!ht]
\centering
\fontsize{8.5}{8.5}\selectfont
\begin{tabular}{lrrrrr}
\toprule
& \multicolumn{2}{c}{Lorenz} & \multicolumn{2}{c}{FitzHugh-Nagumo} \\
\cmidrule{2-5}
Model & w/o $E$ ($\downarrow$) &  w/ $E$ ($\downarrow$) &  w/o $E$ ($\downarrow$) &  w/ $E$ ($\downarrow$) & NFE ($\downarrow$)\\
\midrule
Diffusion~\cite{huang2025improving} & 0.0314 & 0.1001 & 0.0277 & 0.1192 & 128 \\
FM~\cite{huang2025improving} & 0.0348 & 0.0972 & 0.0314 & 0.2164 & 128 \\
FDM~\cite{huang2025improving} & 0.0306 & 0.0914 & 0.0266 & 0.1168 & 128 \\
\midrule
MeanFlow & 0.0469 & 0.1250 & 0.0398 & 0.2268 & 1 \\
MeanFlow & 0.0366 & 0.1011 & 0.0345 & 0.1988 & 8 \\
MeanFlow & 0.0351 & 0.0991 & 0.0302 & 0.1723 & 32 \\
RMFlow (\textbf{ours}) & 0.0332 & 0.0956 & 0.0289 & 0.1543 & 1 \\
\bottomrule
\end{tabular}
\vspace{-0.5em}
\caption{
TV distance between the generated (by different models) and test trajectory distributions, estimated from histogram-based density approximations, with/without conditioning on the event.
}
\label{tb:dynamic_tv}
\end{table}

\begin{table}[!ht]
\centering
\fontsize{8.5}{8.5}\selectfont
\begin{tabular}{lrrrrr}
\toprule
& \multicolumn{2}{c}{Lorenz} & \multicolumn{2}{c}{FitzHugh-Nagumo} \\
\cmidrule{2-5}
Model & w/o $E$ ($\downarrow$) &  w/ $E$ ($\downarrow$) &  w/o $E$ ($\downarrow$) &  w/ $E$ ($\downarrow$) & NFE ($\downarrow$)\\
\midrule
Diffusion~\cite{huang2025improving} & 0.0056 & 0.2774 & 0.0260 & 0.3011 & 128 \\
FM~\cite{huang2025improving} & 0.0081 & 0.2560 & 0.0280 & 0.3468 & 128 \\
FDM~\cite{huang2025improving} & 0.0049 & 0.3045 & 0.0280 & 0.2084 & 128 \\
\midrule
MeanFlow & 0.0109 & 0.3887 & 0.0347 & 0.3921 & 1 \\
MeanFlow & 0.0091 & 0.3163 & 0.0297 & 0.2422 & 8 \\
MeanFlow & 0.0054 & 0.2722 & 0.0281 & 0.2490 & 32 \\
RMFlow (\textbf{ours}) & 0.0059 & 0.2866 & 0.0287 & 0.2499 & 1 \\
\bottomrule
\end{tabular}
\vspace{-0.5em}
\caption{
KL divergence between the generated (by different models) and test trajectory distributions, estimated from histogram-based density approximations, with/without conditioning on the event.
}
\label{tb:dynamic_kl}
\end{table}

\subsection{Text-to-Image}\label{subsec:text2image}
In this experiment, we train MeanFlow and RMFlow for text-to-image generation on the COCO  dataset~\cite{chen2015microsoft}. Following Stable Diffusion~\cite{rombach2022high}, all operations are performed in the latent space $\sR^{4\times 32\times 32}$. The mapping $\phi_{\omega}(\vc)$ converts the text conditions into initial latent states. Concretely, we fine-tune the pretrained text-embedding model \texttt{e5-base}~\citep{wang2022text} and attach an MLP to project the embeddings into the latent space. Additionally, we fine-tune the Stable Diffusion VAE decoder on COCO using PEFT~\citep{hu2022lora,dettmers2023qlora} so that it can decode the final latent state into images. Both MeanFlow and RMFlow use a 480M-parameter U-Net as the latent-space backbone.

We adopt the Karpathy split~\citep{karpathy2015deep} for training and validation, and evaluation is performed with COCO FID-30K following~\cite{rombach2022high,he2025flowtok} 
(details in Appendix~\ref{subsec:experiment-text-to-image}). As shown in Table~\ref{tb:coco}, RMFlow attains FID comparable to the best single-step generators on COCO, such as Distilled Stable Diffusion ~\citep{liu2023instaflow}, StyleGAN-T~\citep{sauer2023stylegan}. 
Importantly, RMFlow (and MeanFlow) is orthogonal to the other methods listed in Table~\ref{tb:coco}, as it does not rely on auxiliary models for training. In contrast, GAN-based approaches require a discriminator, and distilled models depend on a pretrained teacher.
 Moreover, our models are trained under limited computational resources (e.g., RTX 3090/4090 GPUs with 24 GB memory) using mixed-precision bf16, whereas most state-of-the-art models listed in Table~\ref{tb:coco} are trained on multiple A100 80 GB GPUs with full-precision fp16. These results indicate that RMFlow has strong potential for further improvement if trained with larger computational budgets. We also report the performance on CLIP score, see Appendix~\ref{Appendix:clip_score}.

\begin{figure}[!ht]
\centering
\begin{tabular}{cc}
\includegraphics[width=0.45\linewidth,height=0.22\textwidth]{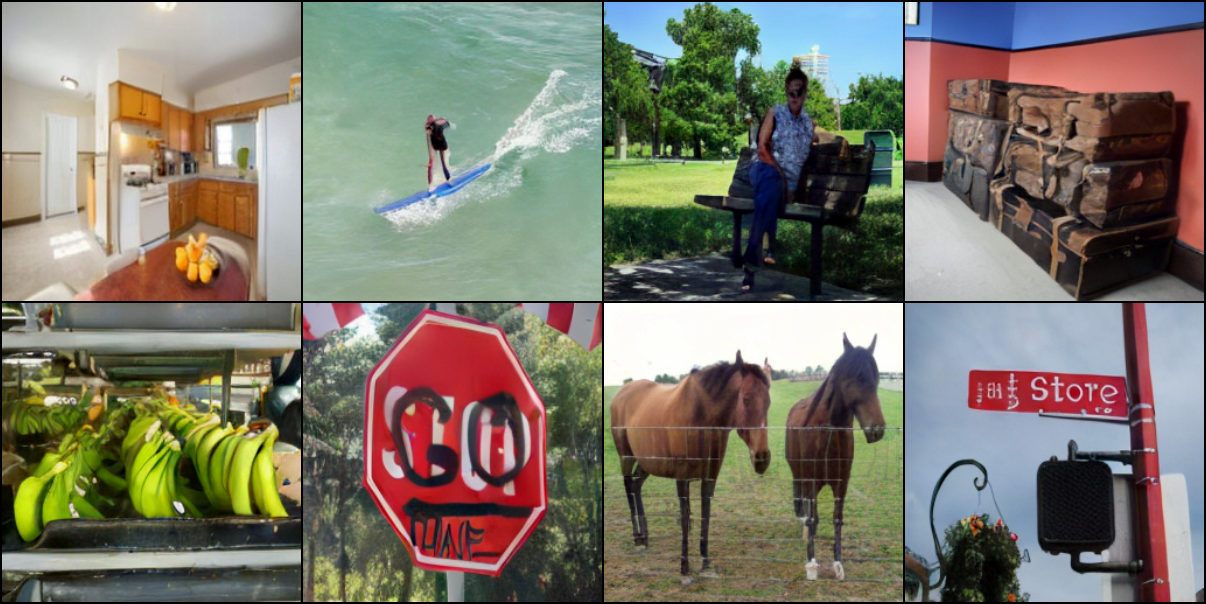}& 
\includegraphics[width=0.45\linewidth,height=0.22\textwidth]{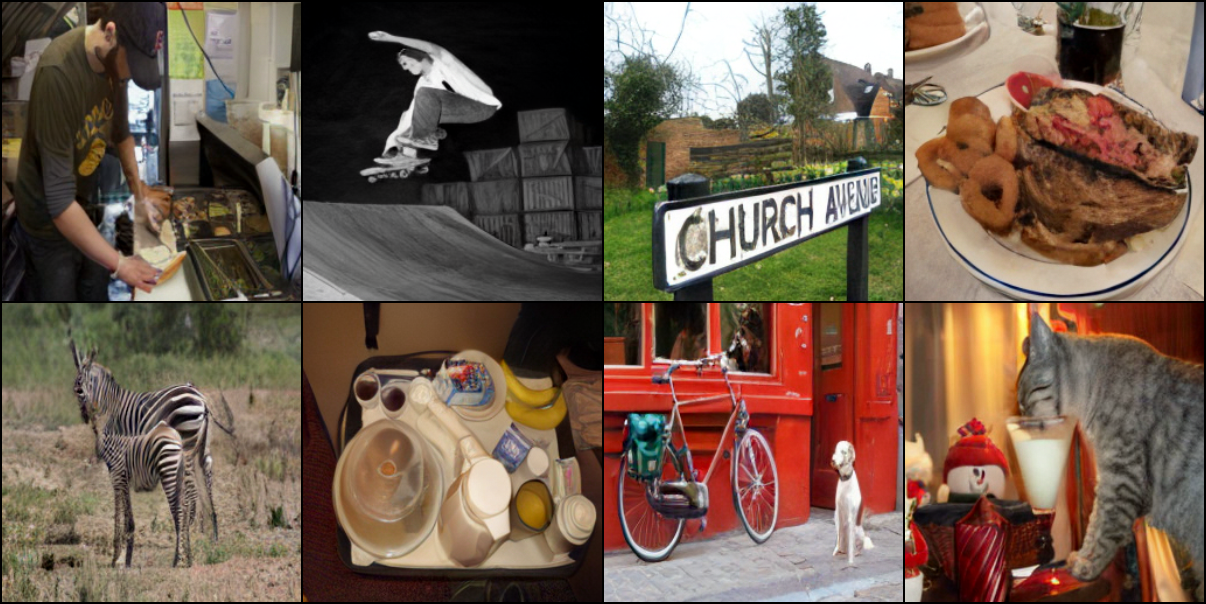}\\
{\fontsize{9}{9}\selectfont
\begin{tabular}{l}
(1) The dining table near the kitchen has a ... \\
(2) A woman riding a surfboard on a wave in the ... \\
(3) A woman sitting on a wooden park bench ... \\
(4) A stack of old trunks and luggage against ...\\
(5) Rows of unripe bananas on display in ...\\
(6) stop sign with spray painted words on it. \\
(7) A couple of horses that are next to a fence.\\
(8) A red and white street sign mounted on ...\\
\end{tabular}} & {\fontsize{8}{8}\selectfont
\begin{tabular}{l}
(1) A man making a sandwich on a lunch truck. \\
(2) A skateboarder performing a trick on an indoor ramp. \\
(3) There is a large sign that says a street name on it. \\
(4) A white plate topped with onion rings and ...\\
(5) A big zebra and a little zebra standing and looking.\\
(6) The meal is ready on the tray to be eaten. \\
(7) A bike and a dog on the sidewalk outside a ...\\
(8) A cat up on a desk drinking milk from a glass.\\
\end{tabular}}
\end{tabular}
\vspace{-0.3em}
\caption{
COCO dataset samples generated using 1-NFE RMFlow conditioned on different input prompts.}
\label{fig:coco}
\end{figure}

\begin{table}[!ht]
\fontsize{6.5}{6.5}\selectfont
\centering
\begin{tabular}{lcccccc}
\toprule
      & Type & params  & NFE & Teacher-free & COCO FID-30K ($\downarrow$) & Resolution\\
& & & & (or discriminator-free) & & \\
\midrule
Stable Diffusion v1.5~\cite{rombach2022high} & Diff & 860M &  $\gg 1$ &\ding{51} & 9.62 & $256\times 256$\\
Stable Diffusion v2.1~\cite{rombach2022high} & Diff & 860M  &  $\gg 1$&\ding{51}  & 13.45& $256\times 256$\\
FlowTok-XL~\cite{he2025flowtok} & ODE & 698M & $\gg 1$&\ding{51}  & 10.1& $256\times 256$\\
Show-o~\cite{xie2024show} & Diff & 1.3B  & $\gg 1$&\ding{51}  & 9.24& $256\times 256$ \\
PixArt~\cite{chen2023pixart} & ODE & 630M & $\gg 1$&\ding{51}  & 7.32& $256\times 256$ \\
LDM~\cite{rombach2022high} & Diff & 1.4B & $\gg 1$&\ding{51} & 12.63& $256\times 256$ \\
\midrule
VQGAN+T~\cite{jahn2021high} & GAN & 1.1B & 1 &\ding{55}& 32.76& $256\times 256$ \\
LAFITE~\cite{zhou2022towards} & GAN & 75M & 1 &\ding{55}& 26.94& $256\times 256$ \\
StyleGAN-T~\cite{sauer2023stylegan} & GAN & 1B & 1 &\ding{55}& 13.90& $256\times 256$ \\
InstaFlow~\cite{liu2023instaflow} & ODE & 900M & 1 &\ding{55}& 13.10& $512\times 512$ \\
UFOGen~\cite{xu2024ufogen} & Diff & 900M & 1 &\ding{55}& 12.78& $512\times 512$ \\
Stable Diffusion + Distill~\cite{liu2023instaflow} & Diff & 900M & 1 & \ding{55}& 34.6 & $256\times 256$ \\
Rectified Flow + Distill~\cite{liu2023instaflow} & ODE & 900M & 1 & \ding{55}& 20.0 & $256\times 256$ \\
\midrule
MeanFlow & ODE & 620M  & 1 &\ding{51}& 27.31& $256\times 256$ \\
RMFlow ({\bf ours}) & Diff & 620M  & 1 &\ding{51}& 18.91& $256\times 256$ \\

\bottomrule
\end{tabular}
\vspace{-0.5em}
\caption{
FID of the generated images on the benchmark COCO2014 dataset using different models.}
\label{tb:coco}
\end{table}

\section{Concluding Remarks}
In this work, we introduce RMFlow, a refinement of MeanFlow with minimal computational and memory overhead. The central innovation lies in augmenting the 1-NFE MeanFlow with a subsequent noise injection step, which facilitates likelihood maximization. To support this mechanism, we propose a novel loss function that jointly minimizes the discrepancy between the exact and learned probability paths while maximizing likelihood. Empirical results demonstrate that 1-NFE RMFlow achieves strong performance in multimodal generation tasks.

A promising direction for future research is to extend RMFlow to support multiple mean flow transport steps. Specifically, we envision applying a noise-injection step after each transport step, which would require the design of a corresponding loss function to maintain likelihood maximization. This extension presents additional challenges compared to the current formulation and opens avenues for more expressive and accurate generative modeling. Another limitation of RMFlow is that it uses a fixed parameter $\sqrt{\sigma_{\min}^2 - \sigma^2}$ during the noise injection step, which may be suboptimal. As future work, we plan to explore more adaptive strategies for selecting this parameter, such as making it learnable or following a dynamic schedule.

\clearpage
\clearpage
\section{Acknowledgement}
This material is based on research sponsored by NSF grants DMS-2208361, DMS-2219956, DMS-2436343, and DMS-2436344, and DOE grants DE-SC0023490, DE-SC0025589, and DE-SC0025801.

\section*{Ethics Statement}
In this paper, we propose a new framework to improve MeanFlow for efficient data generation. The new model can generate high-fidelity data efficiently. Our work belongs to fundamental research and is expected to improve existing models for generative modeling. Our work is methodological, and we validate our proposed approaches on the benchmark datasets. We do not expect to cause negative societal problems. Furthermore, we do not see any issues with potential conflicts of interest and sponsorship, discrimination/bias/fairness concerns, privacy and security issues, legal compliance, and research integrity issues (e.g., IRB, documentation, research ethics.

\section*{Reproducibility Statement}
We are committed to conducting reproducible research. To ensure the integrity and transparency of our work, we employ a multifaceted approach: First, we meticulously compare the novelty of our research against existing literature. This involves a thorough examination of the current state of the field to identify gaps in knowledge and demonstrate the unique contributions of our work. Second, we provide detailed derivations of our proposed approaches and theoretical results. By carefully outlining the mathematical underpinnings of our methods, we enhance the understanding of our work and facilitate its verification by others. Third, we conduct rigorous experiments using widely recognized benchmark datasets. This allows us to evaluate the performance of our methods against established standards and provides a solid foundation for comparison with other approaches. Fourth, we meticulously report experimental details, including the specific datasets used, parameters chosen, and evaluation metrics employed.
Finally, we make all experimental codes, accompanied by comprehensive documentation, publicly available. This open-source approach empowers researchers to inspect our methods, verify our results, and build upon our work. By sharing our code, we foster collaboration, advance the field, and contribute to the overall reproducibility of scientific research.

\newpage
\appendix

\section{Technical Proofs}

\thmKL*

\begin{proof}[proof of Theorem~\ref{thm:KL}]
We begin with the marginal likelihood:
\begin{equation}\label{eq:elbo}
\begin{aligned}
\log p_\theta(\vx_{\rm tgt}) &= \log \mathbb{E}_{\vx_0} \big[p_\theta(\vx_{\rm tgt}|\vx_0)) \big]\\
&\geq \mathbb{E}_{\vx_0}\big[\log p_\theta(\vx_{\rm tgt}|\vx_0)\big],
\end{aligned}
\end{equation}
where the inequality follows from Jensen's inequality.  

Taking expectation over $\vx_{\rm tgt}$ gives
\begin{equation}
\begin{aligned}
\mathbb{E}_{\vx_{\rm tgt}} [\log p_\theta(\vx_{\rm tgt})]
&\geq \mathbb{E}_{\vx_0, \vx_{\rm tgt}}\big[\log p_\theta(\vx_{\rm tgt}|\vx_0)\big].
\end{aligned}
\end{equation}

Now, by substituting the log-likelihood expression \eqref{eq:log-likelihood}, we obtain
\begin{equation}
\begin{aligned}
\mathbb{E}_{\vx_0, \vx_{\rm tgt}}\big[\log p_\theta(\vx_{\rm tgt}|\vx_0)\big]
& =\mathbb{E}_{\vx_0, \vx_{\rm tgt}}\Big[ - \frac{1}{2 (\sigma_{\min}-\sigma)^2} \big\| \vx_{\rm tgt}  - \big( \vx_0 + \hat{\mathbf{u}}_{0,1}(\vx_0; \theta) \big) \big\|^2 + C\Big]
\\
& =  -\frac{1}{2 (\sigma_{\min}-\sigma)^2}\mathbb{E}_{\vx_0, \vx_{\rm tgt}}\Big[ \big\| \vx_{\rm tgt}  - \big( \vx_0 + \hat{\mathbf{u}}_{0,1}(\vx_0; \theta) \big) \big\|^2 \Big]+ C
\\
& =  -\frac{1}{2 (\sigma_{\min}-\sigma)^2}\gL_{\rm NLL} + C
\end{aligned}
\end{equation}

Combining the inequalities, there exist constants $A, C > 0$ such that
\begin{equation}
    -A \cdot \gL_{\rm NLL} + C \;\leq\; \mathbb{E}_{\vx_{\rm tgt}} [\log p_\theta(\vx_{\rm tgt})].
\end{equation}

Finally, recall that
\begin{equation}
    \mathbb{E}_{\vx_{\rm tgt}} [\log p_\theta(\vx_{\rm tgt})] 
    = -H(p_{\rm tgt}) - D_{\rm KL}(p_{\rm tgt}\,\|\,p_\theta),
\end{equation}
which establishes the desired relation.

\end{proof}

\section{Experiment Setup}\label{sec:experiment_setup}

\paragraph{Flow map design.} 
For a given pair $\vz = (\vx_0, \vx_1)$, we choose the conditional velocity field $\vu_t(\vx|\vz)$ following \citep{albergo2023building}, i.e.,
\begin{equation}\label{eq:velocity}
\vu_t(\vx|\vz) = \frac{\dot\gamma(t)}{\gamma(t)}\big(\vx-t\vx_1-(1-t)\vx_0 \big) + (\vx_1-\vx_0),
\end{equation}
where $\gamma(t) = \eta (1-t)$ with $\gamma(0)=\eta$, $\gamma(1)=0$, and $\eta=10^{-2}, 5\times 10^{-2}, 10^{-1}$.

\textbf{Training Setup}: 

\textit{Loss Metric}: Following~\cite{geng2025mean}, we focus on part I (mean flow loss) of $\gL_{\rm RMFlow}$~\ref{eq:RMFlow-objective}, expressed as $\gL=|\Delta|^{2\zeta}_2$, where $\Delta$ denotes the regression error. In practice, we apply a weight $w=\Big(\frac{1}{\|\Delta\|^{2}_2+1e-3}\Big)^m$ with $m = 1-\zeta$. When $m=0.5$, this formulation becomes closely related to the Pseudo-Huber loss introduced in~\cite{song2023improved}; hence, we adopt $m=0.5$ for all experiments. The hyperparameters $\lambda_1$ and $\lambda_2$ in $\gL_{\rm RMFlow}$ are selected individually for each experiment, as reported below.

\textit{Time sampling and condition}: For the part III and II in $\gL_{\rm RMFlow}$~\ref{eq:RMFlow-objective}, we only need to sample $\vx_0$ so that the time is always zero. Similar one sampling method used in~\cite{geng2025mean}, we sample $(t,r)$ such that $p(t) = 2t$ and 
$$p(r|t)=q\cdot\frac{\bm{1}_{0\leq r < t}}{t} + (1-q)\delta(r-t)$$
so for given sampled $t$, we sample $r$ from $\gU[0, t)$ with probability $q$ and set $r=t$ probability $1-q$. $q$ is selected individually for each experiment, as reported below. We use positional embedding for $(r,t)$, which are then combined and provided as the conditioning of the neural network. As used in~\cite{geng2025mean}, it is not necessary for the network to directly condition on (r, t), so we have $\vu_{t,r}(\cdot;\theta):=\texttt{net}(\cdot,t,t-r)$.

\subsection{Ablation study}\label{subsec:experiment-Synthetic}
In this task, we focus on how the value of $\lambda_1$ balances the Wasserstein (part {\text I} in~\eqref{eq:RMFlow-objective}) and the likelihood (part {\text II} in~\eqref{eq:RMFlow-objective}).
\subsubsection{Gaussian Mixture}
Table~\ref{tb:ablation_mixG} shows that the RMFlow has the best performance when $\lambda_1 = 1e-1,1e-2$.
\begin{table}[!ht]
\centering
\begin{tabular}{cccccc}
\toprule
$\lambda_1$ & 0 & 1e-2 & 1e-1 & 1e1 & 1e2 \\
\midrule
&\includegraphics[width=0.15\linewidth]{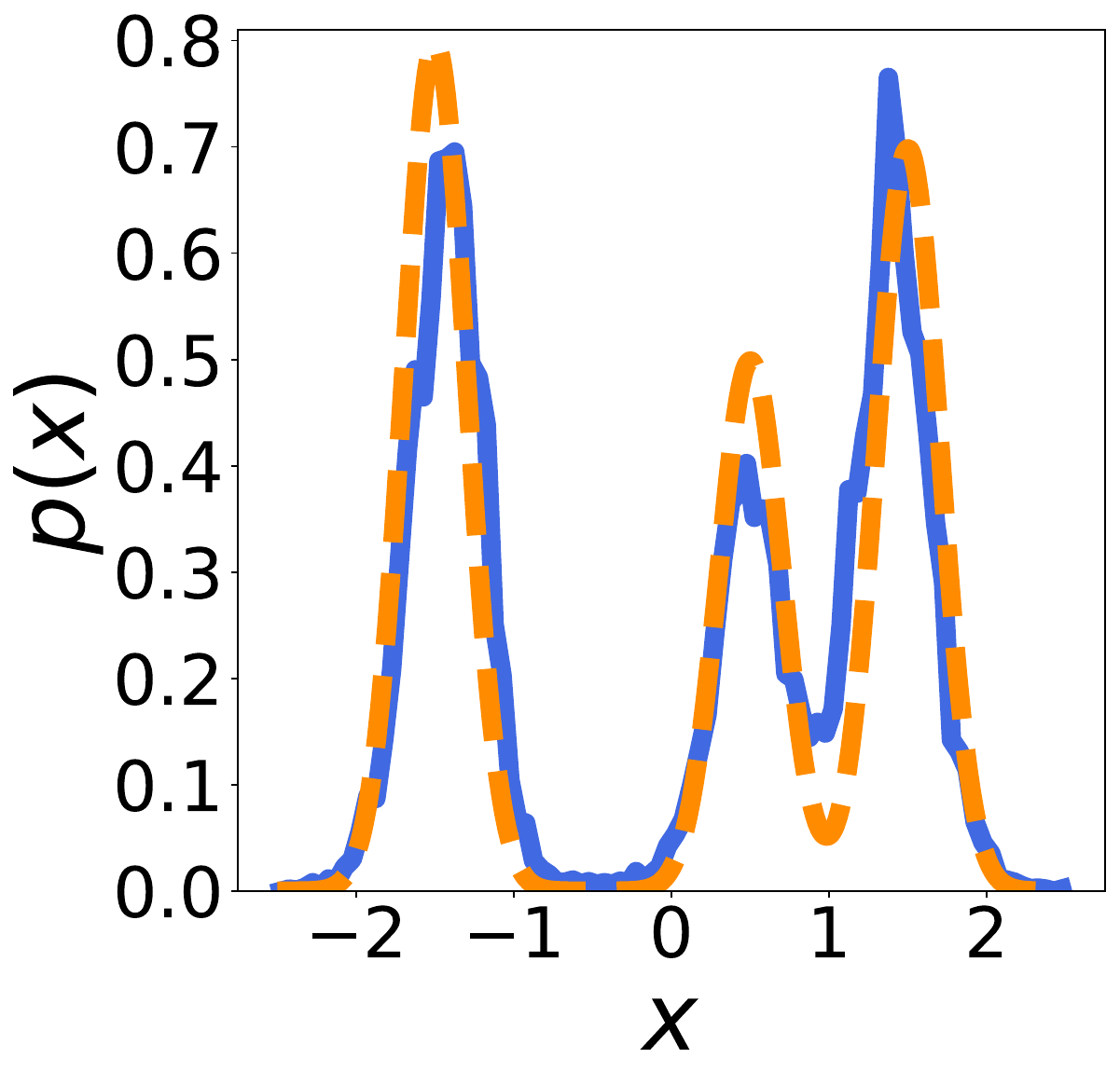} 
& 
\includegraphics[width=0.15\linewidth]{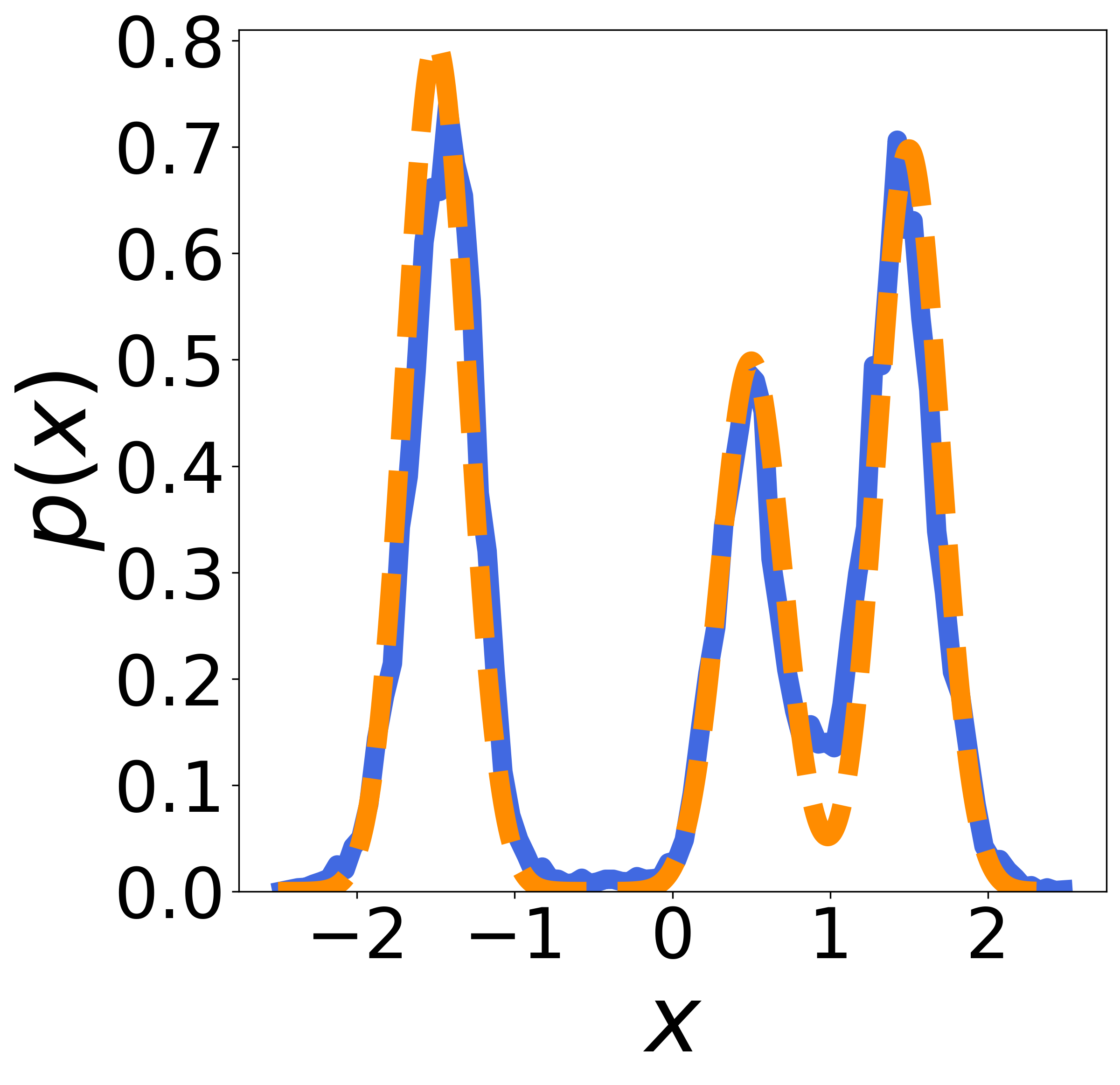} & 
\includegraphics[width=0.15\linewidth]{figs/mixG/ablation/x_1_prime.png} & 
\includegraphics[width=0.15\linewidth]{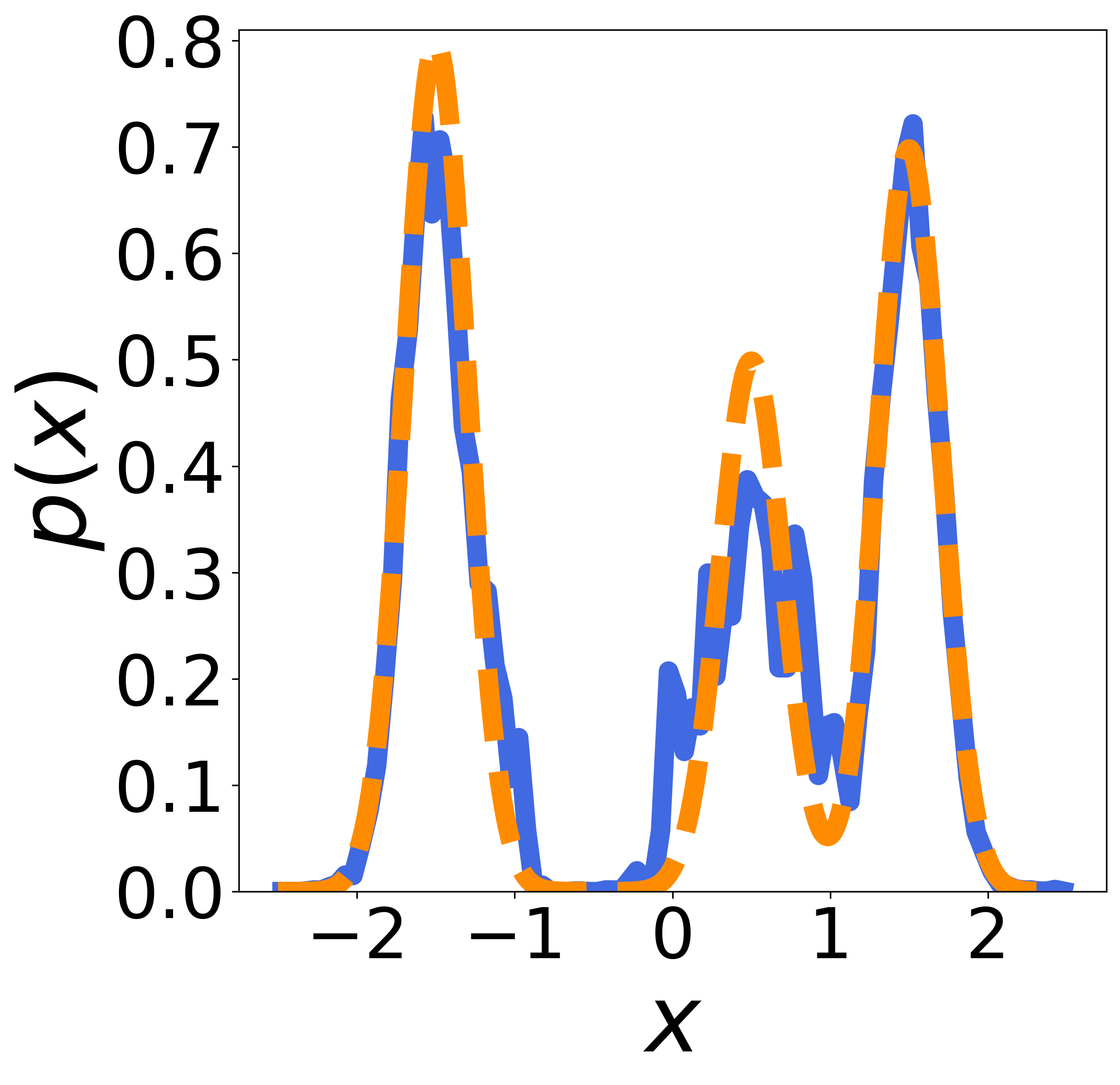} & 
\includegraphics[width=0.15\linewidth]{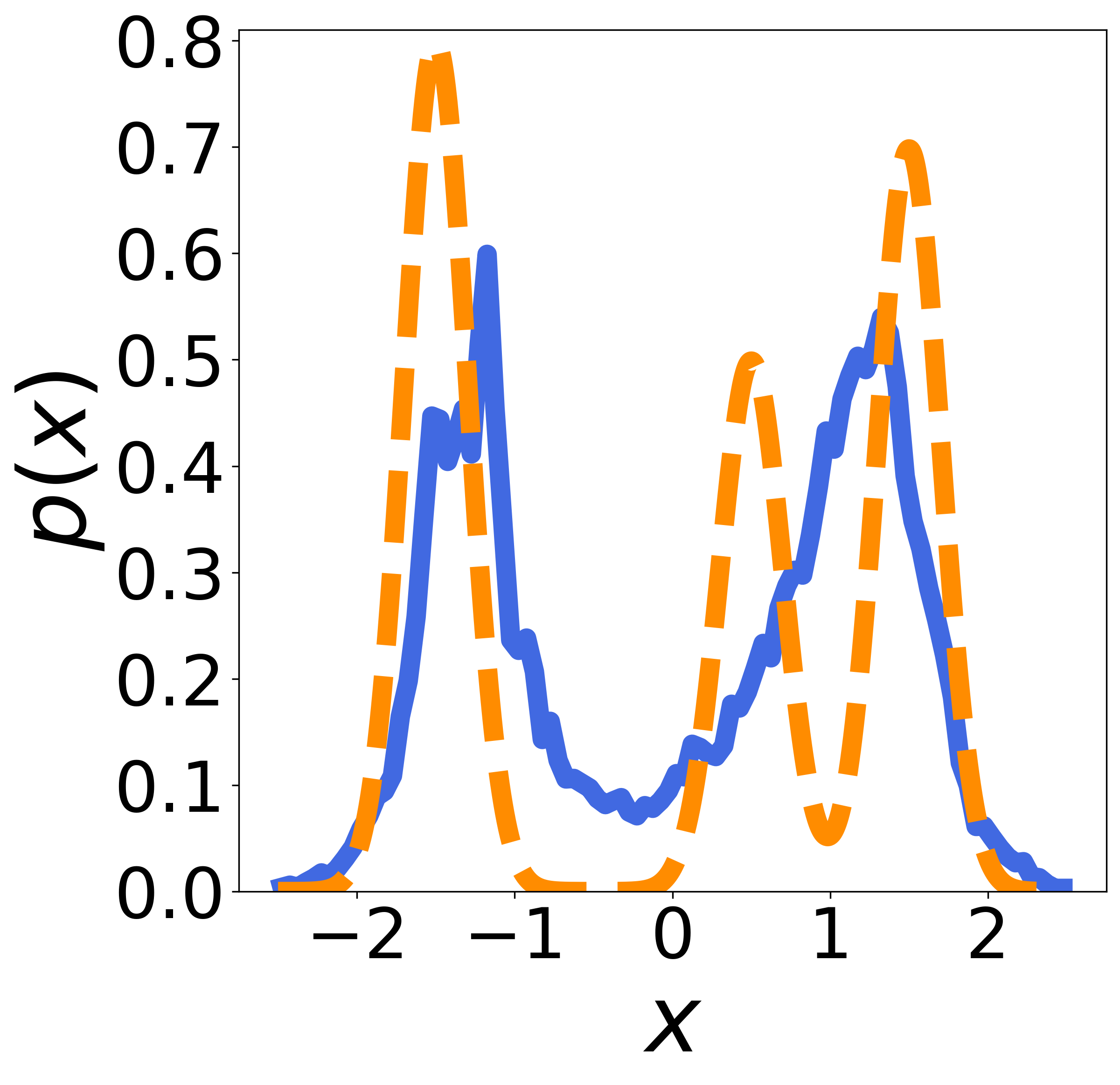}\\
\bottomrule
\end{tabular}
\caption{\footnotesize Ablation study on the Gaussian mixture task with varying $\lambda_1$.}
\label{tb:ablation_mixG}
\end{table}
\subsection{Checkerboard}
Table~\ref{tb:ablation_checkerboard} shows that the RMFlow has the best performance when $\lambda_1 = 1e-1$.
\begin{table}[!ht]
\centering
\begin{tabular}{cccccc}
\toprule
$\lambda_1$ & 0 & 1e-2 & 1e-1 & 1e1 & 1e2 \\
\midrule
TV & 0.238 & 0.201 & 0.173 & 0.222 & 0.289 \\
KL & 0.311 & 0.228 & 0.163 & 0.237 & 0.425\\
\bottomrule
\end{tabular}
\caption{\footnotesize Ablation study on the 2D Checkerborad synthetic task with varying $\lambda_1$.}
\label{tb:ablation_checkerboard}
\end{table}
\subsection{QM9}
Table~\ref{tb:ablation_qm9} shows that the RMFlow has the best performance when $\lambda_1 = 5e-2$.
\begin{table}[!ht]
\centering
\begin{tabular}{ccccccc}
\toprule
$\lambda_1$ & 0 & 1e-2 & 5e-2 & 1e-1 & 1 & 1e1\\
\midrule
Atomic Stab. & 98.4 & 98.8 & 98.9 & 98.8 & 98.0 & 97.6 \\
Mol Stab. & 84.3 & 92.1 & 93.2 & 92.9 & 87.2 & 83.9\\
\bottomrule
\end{tabular}
\caption{\footnotesize Ablation study on the QM9 molecule generation task with varying $\lambda_1$.}
\label{tb:ablation_qm9}
\end{table}

\subsubsection{Configuration}
The Gaussian mixture and checkerboard experiments share the same configuration, differing only in the model’s input and output dimensions. The configuration is summarized in Table~\ref{tb:Synthetic_config}.
\begin{table}[!ht]
\centering
\begin{tabular}{lcc}
\toprule
\textbf{Model} &number of layers & 6 \\
&hidden dim & 256 \\
&activation & SiLU \\
\midrule
\textbf{Train}& iteration & 1e5  \\
& batch size & 256 \\
& optimizer & Adam~\cite{diederik2014adam} \\
& lr schedule & polynomial \\
& lr & 1e-4 \\
& Adam($\beta_1$, $\beta_2$) & (0.9, 0.95) \\
& ema decay & 0.9995 \\
& precision & fp32 \\
& $\lambda_1$ & [1e-1, 1e-2]\\
\bottomrule
\end{tabular}
\caption{\footnotesize Training Setup and Backbone Configuration for MeanFlow and RMFlow on Synthetic Tasks}
\label{tb:Synthetic_config}
\end{table}

\subsection{Text-to-Image}\label{subsec:experiment-text-to-image}
\textbf{VAE Decoder}: We use the VAE in \texttt{stabilityai/sd-vae-ft-mse}~\cite{rombach2022high} and fine-tune the \texttt{conv-in}, \texttt{up-blocks} (0,-1), \texttt{conv-out}, \texttt{quant-conv}, and \texttt{post-quant-conv} parts of the VAE decoder to refine the decoded image quality on the target dataset using the VAE decoder reconstruction loss~\cite{kingma2013auto} with the learning rate $1e-5$.

\textbf{Text Embedding Model}: Following Section 5.4, we optimize solely the MLP parameters $\omega$ that map pretrained \texttt{e5} embeddings to the latent space under \eqref{eq:RMFlow-objective}. Accordingly, $\phi_{\omega}$ is implemented as a fixed \texttt{e5} encoder plus a trainable MLP (i.e., a PEFT setup).

\textbf{U-Net Backbone and Training Setup}
We build the latent-space backbone for Mean Flow and RMFlow by reusing selected U-Net blocks ($\sim$480M parameters) from pretrained Stable Diffusion~\cite{rombach2022high}, with an additional time embedding for $r$. This is effectively transfer learning. Table~\ref{tb:text2img_config} shows the U-Net configuration and the training setup.

\begin{table}[!ht]
\fontsize{7.5}{7.5}\selectfont
\centering
\begin{tabular}{llc}
\toprule
\textbf{U-Net}
& block-out-channels   & [320, 640, 1280, 1280] \\
& down-block types     & [CrossAttnDownBlock2D, CrossAttnDownBlock2D, \\
&                      & \hspace{4em} CrossAttnDownBlock2D, DownBlock2D] \\
& layers-per-block     & 2 \\
& attention-head-dim   & 8 \\
& cross-attention-dim  & 768 \\
\midrule 
\textbf{Pre-train} & loss & $\gL_{\rm CMFM}$ \\
& epochs & 500  \\
& batch size per GPU & 16 \\
& optimizer & Adam \\
& lr schedule & polynomial \\
& Warm up epoch & 2 \\
& lr & 1e-4 \\
& Adam($\beta_1$, $\beta_2$) & (0.9, 0.95) \\
& ema decay & 0.9995 \\
& precision & fp16 \\
& trainable param & 480M \\
\midrule
\textbf{Post-train} & loss & $\gL_{\rm RMFlow}$ \\
& epochs & 500  \\
& batch size per GPU & 16 \\
& optimizer & Adam \\
& lr schedule & polynomial \\
& Warm up epoch & 10 \\
& lr & 5e-5 \\
& Adam($\beta_1$, $\beta_2$) & (0.9, 0.95) \\
& ema decay & 0.9995 \\
& precision & bf16 \\
& trainable param & 210M \\
&$\lambda_1$ & [5e-2, 1e-2] \\
\midrule
Reg for $\phi_{\omega}$ & $\lambda_2$ & 1e-4 \\
Time sample & $p(r\neq t)$ & 0.25 \\
\bottomrule
\end{tabular}
\caption{\footnotesize Training Setup and Backbone Configuration for MeanFlow and RMFlow on the COCO Text-to-Image Dataset}
\label{tb:text2img_config}
\end{table}

\subsection{Context-to-Molecule}\label{subsec:experiment-context-to-molecule}
\textbf{Context Embedding Model}: We implement $\phi_{\omega}$ as an MLP followed by a single EGNN layer. The MLP projects the context vector into the data space, and the EGNN layer further refines these representations. The EGNN configuration matches that used in the Mean Flow/RMFlow backbone.

\textbf{EGNN Backbone and Training Setup}
We use the same EGNN Backbone in~\citep{hoogeboom2022equivariant} augmented with a time-embedding module for the additional scalar time variable $r$. Table~\ref{tb:cond2mole_config} shows the training setup and configuration.

\begin{table}[!ht]
\fontsize{7.5}{7.5}\selectfont
\centering
\begin{tabular}{llc}
\toprule
\textbf{EGNN}
& number of layers   & 9 \\
& acitivation     & SiLU \\
& hidden dim     & 256 \\
\midrule 
\textbf{Pre-train} & loss & $\gL_{\rm CMFM}$ \\
& epochs & 1500  \\
& batch size per GPU & 64 \\
& optimizer & Adam \\
& lr schedule & polynomial \\
& Warm up epoch & 10 \\
& lr & 1e-4 \\
& Adam($\beta_1$, $\beta_2$) & (0.9, 0.95) \\
& ema decay & 0.9995 \\
& precision & fp32 \\
\midrule
\textbf{Post-train} & loss & $\gL_{\rm RMFlow}$ \\
& epochs & 1500  \\
& batch size per GPU & 64 \\
& optimizer & Adam \\
& lr schedule & polynomial \\
& Warm up epoch & 10 \\
& lr & 1e-4 \\
& Adam($\beta_1$, $\beta_2$) & (0.9, 0.95) \\
& ema decay & 0.9995 \\
& precision & fp16 \\
&$\lambda_1$ & [1e-2, 5e-2] \\
\midrule
Reg for $\phi_{\omega}$ & $\lambda_2$ & 1e-4 \\
Time sample & $p(r\neq t)$ & 0.5 \\
\bottomrule
\end{tabular}
\caption{\footnotesize Training Setup and Backbone Configuration for MeanFlow and RMFlow on Context-to-Molecule Generation on QM9 Dataset}
\label{tb:cond2mole_config}
\end{table}

\subsection{Time Series: Dynamic System}\label{subsec:experiment-dynamic-system}
\textbf{Trajectory dataset}: we use the same dataset as decribed in \cite[Appendix B.1]{huang2025improving}.

\textbf{Models}: (1) We implement the guidance embedding function $\phi_{\omega}$ as an MLP that maps the event-guidance vector together with the first three states $\vx(\tau_1), \vx(\tau_2), \vx(\tau_3)$ into the prior sample; (2) we adopt the UNet architecture from~\cite{finzi2023user}, adding an additional time embedding for $r$.

\textbf{Training Setup}: see Table~\ref{tb:dynamic_config}
\begin{table}[!ht]
\fontsize{7.5}{7.5}\selectfont
\centering
\begin{tabular}{lcc}
\toprule
\textbf{Train}& iteration & 1e5  \\
& batch size & 500 \\
& optimizer & Adam \\
& lr & 1e-4 \\
& weight decay & 0.995 \\
&$\lambda_1$ & 1e-1\\
&$\lambda_2$ & 1e-4\\
\bottomrule
\end{tabular}
\caption{\footnotesize Training Setup for MeanFlow and RMFlow on Dynamical System Forecasting Tasks}
\label{tb:dynamic_config}
\end{table}

\subsection{Additional Experiments}\label{Appendix:clip_score}
In this section, we report the comparison of CLIP score of our RMFlow, MeanFlow, and some models on the COCO dataset (2017 version dataset splitting, with 5000 images, following~\cite{liu2023instaflow}). See Table~\ref{tb:coco_clip}.
\begin{table}[!ht]
\fontsize{6.5}{6.5}\selectfont
\centering
\begin{tabular}{lcccccc}
\toprule
& Type & params  & NFE & Teacher-free & FID-5k & Clip ($\uparrow$)\\
& & & & (or discriminator-free) & \\
\midrule
Stable Diffusion v1.5~\cite{rombach2022high} & Diff & 860M &  $\gg 1$ &\ding{51} & 20.1 & 0.315 \\
\midrule
InstaFlow~\cite{liu2023instaflow} & ODE & 900M & 1 &\ding{55}& 23.4 & 0.304\\
StyleGAN-T~\cite{sauer2023stylegan} & GAN & 1B & 1 &\ding{55}& 24.1 & 0.305 \\
PD-SD~\cite{meng2023distillation} & Diff & N/A & 1 &\ding{55}& 37.2 & 0.275\\
\midrule
MeanFlow & ODE & 620M  & 1 &\ding{51} & 38.5 & 0.273 \\
RMFlow ({\bf ours}) & Diff & 620M  & 1 &\ding{51}& 27.9 & 0.291 \\
\bottomrule
\end{tabular}
\vspace{-0.9em}
\caption{\footnotesize CLIP scores of the generated images on the benchmark COCO2017 dataset using different models.}
\label{tb:coco_clip}
\end{table}

\subsection{Policy Gradient with Physical Feedback}\label{sec:rlpf_loss}
We define a reward $r(\hat\vx_{\rm tgt})$ to quantify the molecule stability of the generated sample $\hat\vx_{\rm tgt}$ on QM9 dataset, then define the policy gradient following~\cite{black2023training}:
\begin{equation*}
\nabla_{\theta}\gL_{\rm RL}:=\E\Big[\nabla_{\theta}\log p_{\theta}(\vx_{\rm tgt}\mid\vx_0)r(\hat\vx_{\rm tgt})\Big]
\end{equation*}
and the corresponding loss function
\begin{equation*}
\gL_{\rm RL}:=-\E\Big[\log p_{\theta}(\vx_{\rm tgt}\mid\vx_0)r(\hat\vx_{\rm tgt})\Big]   
\label{eq:rlft_loss}
\end{equation*}
where 
\begin{equation*}\label{eq:log-likelihood-2}
\log p_\theta(\vx_{\rm tgt} \mid \vx_0) = -\frac{1}{2 (\sigma_{\min}^2-\sigma^2)} \big\| \vx_{\rm tgt}  - \big( \vx_0 + \hat{\vu}_{0,1}(\vx_0; \theta) \big) \big\|^2 + C,
\end{equation*}
We then perform reinforcement learning fine-tuning by augmenting the RMFlow objective~\ref{eq:RMFlow-objective}:
$$\gL_{\rm RMF+RL}(\theta,\omega)=\gL_{\rm RMFlow}(\theta,\omega) + \eta \gL_{\rm RL}(\theta)$$
where $\eta$ is a hyperparameter.

\end{document}